%% file: arxiv.tex
\documentclass{article} 
\let\ORIGaddcontentsline\addcontentsline
\usepackage{iclr2027_conference,times}
\let\addcontentsline\ORIGaddcontentsline

\usepackage{amssymb,amsmath,amsthm}
\usepackage{mathrsfs}
\usepackage{algorithm}
\usepackage{algpseudocode}
\usepackage{graphicx}
\usepackage{float}
\usepackage{subcaption}
\usepackage{wrapfig2}
\usepackage{svg}
\usepackage{booktabs}
\usepackage{multirow}
\usepackage[table,dvipsnames]{xcolor}
\usepackage{titletoc}
\usepackage{hyperref}
\usepackage{enumitem}
\setlist[itemize]{leftmargin=15pt, noitemsep, topsep=0pt, partopsep=0pt, parsep=0pt}
\usepackage[most]{tcolorbox}

\usepackage{pifont}
\usepackage{listings}
\lstdefinelanguage{systemprompt}{
    basicstyle=\ttfamily\scriptsize,
    breaklines=true,
    breakatwhitespace=true,
    frame=none,
    backgroundcolor=\color{white!98!gray},
    breakindent=0pt,
}
\newcommand{\pinput}[1]{\textcolor{RoyalBlue}{\{#1\}}}

\makeatletter
\renewcommand\paragraph{%
  \@startsection{paragraph}{4}{\z@}%
    {0pt}%
    {-1em}%
    {\normalfont\normalsize\bfseries}%
}
\makeatother

\theoremstyle{definition}
\newtheorem{theorem}{Theorem}[section]
\newtheorem{proposition}[theorem]{Proposition}

\theoremstyle{definition}

\newtheorem*{remark}{Remark}

\definecolor{tomcolor}{RGB}{0, 100, 180} 
\definecolor{toscolor}{RGB}{180, 70, 20} 

\colorlet{ours}{blue!10}
\colorlet{bad}{red!10}
\newcommand{\std}[1]{\tiny$\pm$#1}

\newcommand{\cmark}{\textcolor{ForestGreen}{\ding{51}}}
\newcommand{\xmark}{\textcolor{red!75!black}{\ding{55}}}
\newcommand{\reading}[1]{\textcolor{toscolor}{\textbf{#1}}}

\title{
Theory of Scene: Breaking the Symmetry Trap in Multi-Agent LLM Coordination
}

\author{
Liangqi Yuan$^{1}$\thanks{Corresponding author: Liangqi Yuan (\texttt{liangqiy@purdue.edu})} \quad
Wenzhi Fang$^{1}$ \quad
Shiqiang Wang$^{2}$ \quad
Christopher G. Brinton$^{1}$\\
$^1$Purdue University \quad
$^2$University of Exeter\\
}

\iclrfinalcopy 
\begin{document}

\maketitle

\begin{abstract}
Multi-agent systems built on large language models (LLMs) are largely homogeneous, as their agents behave alike even across distinct LLMs. We show that when such agents act concurrently without communication, they collide on targets they must split and diverge on targets they must take together, a double failure we term the \emph{symmetry trap}. Theory of Mind (ToM), widely used for coordination without communication, cannot escape this trap, since homogeneous agents form the same prediction of one another and respond to it in the same way.
We propose \textbf{Theory of Scene (ToS)}, a training-free reasoning schema in which each agent reads its public role, the only difference between the agents, and the task context they all observe. Homogeneous agents thereby derive one division of labor, each taking the part its role fixes, which turns homogeneity from the cause of the trap into the cure. ToS reads the role together with the scene through \emph{role gating}, which determines whether ownership overlaps or is already divided, and the task context through \emph{task coupling}, which infers whether the team must converge on each target, divide it, or take its stages in turn.
We evaluate on DivvyBench, a controlled environment we introduce, whose target types make an episode Competitive, Cooperative, or Mixed across Tabletop, Airspace, and Household scenarios, and on two established agentic benchmarks, GovSim and Overcooked. ToS outperforms all six baselines on every benchmark, and each baseline falls far behind it in at least one setting. Against ToM given the same inputs, ToS raises the DivvyBench success rate from 71.1\% to 99.6\%, the GovSim total gain from 207 to 400, and the Overcooked level-normalized throughput from 1.41 to 1.67.
\end{abstract}

\begin{figure}[H]
\centering
\vspace{-20pt}
\includegraphics[width=\linewidth]{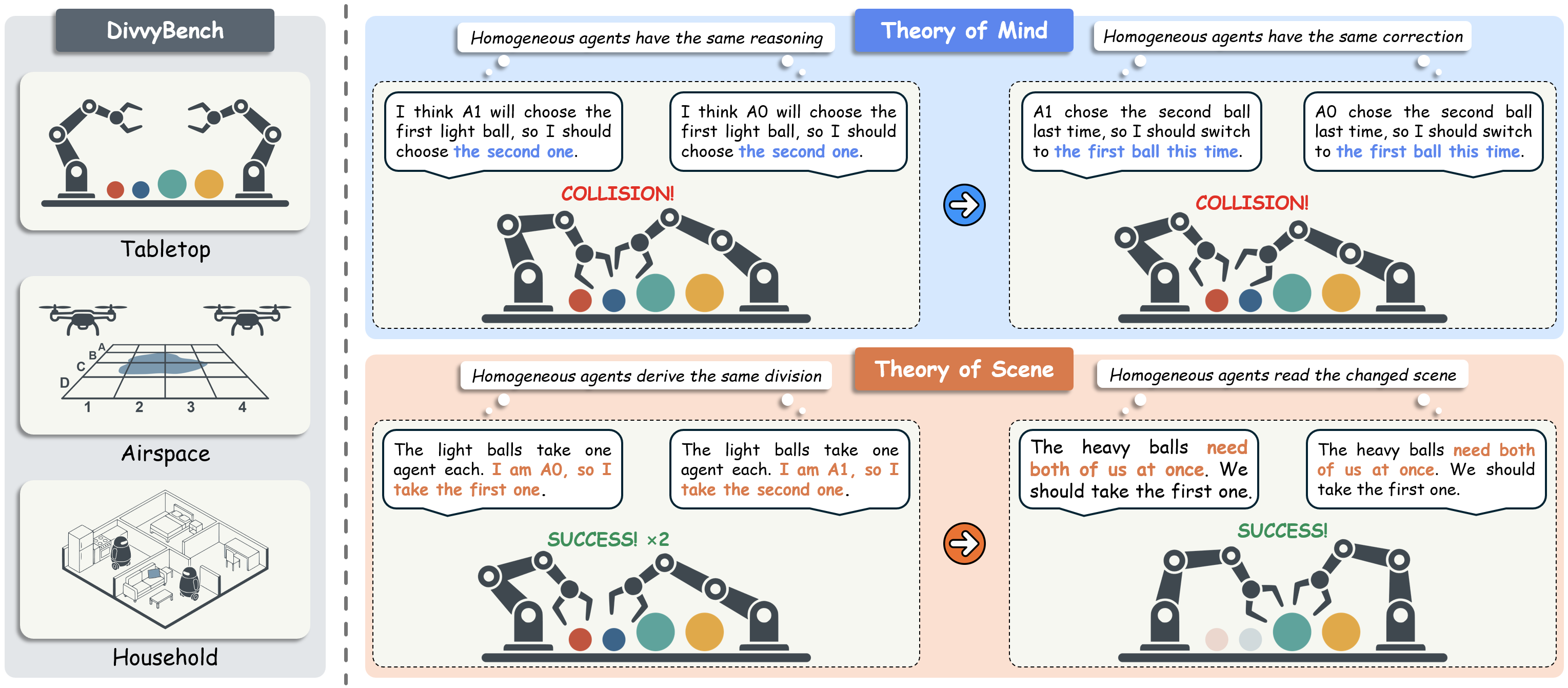}
\vspace{-15pt}
\caption{\textbf{Symmetry trap: ToM vs. ToS.} \textit{Left:} DivvyBench covers three scenarios, Tabletop picking, Airspace survey, and Household assistance. \textit{Top:} Under ToM, homogeneous agents reason identically at every step, so they collide, correct, and collide again. \textit{Bottom:} Under ToS, the agents read the division of labor from the shared scene and the public role, diverging onto the Competitive targets and converging on the Cooperative targets.}
\label{fig:intro}
\end{figure}

\section{Introduction}

Large language model (LLM) agents increasingly work in teams, typically several instances of one LLM~\citep{li2023camel, hong2024metagpt}. Coordination among such teams usually relies on (i) a means for the agents to communicate~\citep{chopra2025ripple, grotschla2025agentsnet, jian2026gated, yang2025agentnet} or (ii) a predefined division of labor among the agents~\citep{hong2024metagpt, qian2024chatdev}. Communication requires resources (e.g., wireless channels), which can be costly, forbidden by policy, and/or unreliable in multi-robot or multi-vehicle teams. A predefined division is realized by prompts that assign each agent a fixed responsibility, such as planner, executor, or verifier, yet the right division changes with the scene at runtime. Lacking both, the agents must coordinate unaided on tasks that often mix Competitive targets, which they must divide, with Cooperative targets, which require several of them at once~\citep{carroll2019utility, agapiou2022melting, piatti2024cooperate}. Yet such agents tend to be homogeneous, behaving alike because even distinct LLMs produce strikingly similar outputs~\citep{jiang2025artificial, kim2025correlated, goel2025great}. They therefore collide on Competitive targets, and any mechanism that separates them there also splits them across Cooperative targets. We identify this double failure as \emph{the symmetry trap, which arises when homogeneous agents act concurrently without communication and are distinguished only by a public role} (Figure~\ref{fig:intro}, top).

Escaping the symmetry trap therefore requires each agent to derive a division of labor by itself, such that all agents arrive at the same division. A public role alone identifies each agent but divides no work, and raising the sampling temperature separates the agents at random, which leaves no division shared. Theory of Mind (ToM), a common approach to multi-agent coordination, has each agent predict the others and best-respond to the prediction~\citep{agashe2025llm, cross2025hypothetical, zhang2024proagent}, yet a homogeneous agent that predicts its teammates recovers only its own reasoning. Conventions and index-based assignments determine which agent takes which target once the team knows it must separate~\citep{schelling1980strategy, ashery2025emergent}, but do not infer from the scene whether a target requires divergence or convergence. Multi-agent reinforcement learning (MARL) breaks symmetry by training over the symmetry group~\citep{hu2020other}, which is unavailable to LLM agents that coordinate at inference. We therefore ask the following question. \emph{How can homogeneous LLM agents escape the symmetry trap by deriving the same division of labor from one shared scene, one that diverges, converges, or takes turns from one target to the next?}

To answer this question, we propose \textbf{Theory of Scene (ToS)}, a training-free reasoning schema that derives the division of labor from the observable scene, whereas ToM predicts the other agents' hidden intent. The intuition is that a scene admits one or more rational divisions of labor, and that any agent can derive one of them from the scene alone. ToS reads the two signals the setting supplies: (i) the \emph{public role}, the only signal that distinguishes homogeneous agents, and (ii) the \emph{task context}, which specifies the work the team must complete. The shared scene thus leads homogeneous agents to the same division, and the public role assigns each agent its own part, so the team diverges on Competitive targets and converges on Cooperative ones without exchanging any message (Figure~\ref{fig:intro}, bottom). Because the interaction is agentic, a disagreement at one step is resolved over the steps that follow. Our contributions are summarized as follows.

\begin{itemize}
\item \textbf{The Symmetry Trap and DivvyBench.} We characterize the \emph{symmetry trap}, in which homogeneous agents collide where the targets must be divided, while the randomness or prediction that separates them also separates them where they must converge. To isolate the trap, we build DivvyBench, a controlled environment that reduces each step to the choice of a target and composes each episode of one target type or a mix of both, across the Tabletop, Airspace, and Household scenarios.
\item \textbf{Theory of Scene.} We propose ToS reasoning, in which each agent reads the shared scene through two mechanisms. (i) \emph{Role gating} settles from the scene and the public role whether ownership is overlapping or already divided, and derives a division only where it is overlapping. (ii) \emph{Task coupling} infers what each target requires of the team, and the team accordingly converges on it, divides it, or takes its stages in turn, which determines the division of labor.
\item \textbf{Empirical Validation.} Three agentic benchmarks, our DivvyBench, GovSim, and Overcooked, show that ToS performs best throughout, while the strongest of the six baselines, each fixed to one reasoning schema, varies across benchmarks. Incorporating ToS's readings into ToM improves ToM yet leaves it below ToS, because ToM derives its action from its prediction, which these readings leave unconstrained.
\end{itemize}

\section{Related Work}

\paragraph{Multi-Agent Coordination: Competitive vs.\ Cooperative.}
Prior work on multi-agent LLM coordination studies competition and cooperation as separate problems and assumes a communication channel. \emph{Competitive} interaction, where the agents contend for the same outcome, is approached through negotiation, debate, and strategic games~\citep{du2024improving, duan2024gtbench}. \emph{Cooperative} interaction, where the agents share a goal, is approached through role-based frameworks that decompose a task and negotiate a plan~\citep{qian2024chatdev, wu2024autogen} and through joint-task benchmarks~\citep{sun2025collab, qian2026collabbench, kim2026teambench}. Each line of work fixes a single setting, and even evaluations spanning both treat them as separate games, so an agent never has to adapt its coordination strategy within an episode. From embodied cooperative suites~\citep{zhang2024building} and open-ended survival worlds~\citep{tessera2026benchmarking} to language games~\citep{li2025systematic, abdelnabi2024cooperation, lupu2025decrypto}, coordination benchmarks provide this channel, through which the agents negotiate the division of labor. Mixed-motive suites in MARL place both settings in one episode~\citep{papoudakis2021benchmarking, agapiou2022melting}, which the agents learn in training. We introduce DivvyBench, the first environment in which (i) the agents are homogeneous and separated only by a public role, (ii) they act simultaneously, before seeing the others' actions, (iii) they have no communication channel, and (iv) each target carries its own demand, so the same agents must diverge and converge within one episode.

\paragraph{Non-Communicating Coordination: ToM and Its Alternatives.}
Coordination without a channel has been approached in three ways, each resting on something this setting removes. The first predicts the other agents and best-responds, as in MARL and Bayesian inverse planning~\citep{wang2022tom2c, wu2021too}. LLM agents with ToM infer the beliefs and intentions behind the other agents' behavior~\citep{li2023theory, mu2026adaptive}, as ProAgent does by predicting a teammate's action~\citep{zhang2024proagent} and HypMinds by testing hypotheses about other agents~\citep{cross2025hypothetical}, both built for heterogeneous teammates. \citet{hayler2026zero} tell the agents their partner is an identical copy and still find no convention settled, and they add a self-check that only rejects flippable conventions. The second is conventions and focal points, which coordinate on a salient option under a commonly known rule fixed before the scene is read~\citep{schelling1980strategy, ashery2025emergent, aharon2026tacit}. The third resolves symmetry with MARL techniques, randomizing over the symmetry group during training~\citep{hu2020other}, applying off-belief learning~\citep{hu2021off}, or sampling a rank ordering at execution~\citep{patil2026randomness}. We therefore base ToS on the two signals that the setting provides, the public role and the task context, which removes the need for (i) predicting the other agents, (ii) a division of labor fixed before the scene is read, and (iii) training on the task or randomness at execution.

\section{The Symmetry Trap}
\label{sec:background}

\paragraph{The DivvyBench Environment.}
DivvyBench distills existing coordination benchmarks to the one decision their LLM agents make, the choice of a target, and discards the simulator that executes it. In C-WAH~\citep{zhang2024building} the LLM chooses which room to search and in Overcooked~\citep{carroll2019utility, sun2025collab} which station to use, and the surrounding simulator confounds the outcome, since a team that underperforms may have failed to coordinate or may simply have been overwhelmed by the scene. In DivvyBench, homogeneous agents must take every target in a shared scene (Appendix~\ref{appendix:divvybench}), acting at the same time without any message and deciding only which target to select. Scoring is all-or-nothing, so an episode counts only when every target is taken within the budget, and each collision or stuck step spends part of it. The same rules run in three scenarios, in which the agents collect colored balls in \emph{Tabletop}, survey grid areas in \emph{Airspace}, and search indoor locations in \emph{Household}, and we say that an agent takes a target when it completes that operation. A Competitive target is taken only when exactly one agent selects it, so agents that select the same one collide and take nothing, and a team that diverges completes the task in half the steps one agent needs alone. A Cooperative target is taken only when every agent selects it in the same step, so agents that split across targets take nothing. A scene that holds both types is Mixed, where the two demands fall on the same agents within one episode. The rest of this section runs on Tabletop, with every agent an instance of one LLM under one prompt, and holds the two pure settings apart, since a mixed scene superposes them and leaves each failure unattributable.

\paragraph{The Public Role and Temperature Cannot Break Competitive Symmetry.}
The first candidate for breaking the symmetry is the public role alone, and Figure~\ref{fig:divvybench_motivation} (left) shows that w/o Role (identical prompts) and w/ Role (a fixed public role, Leader or Follower) both collapse in the Competitive setting, where neither completes any episode at temperature $0$. A label such as Leader identifies an agent without specifying which targets belong to it, so every agent reasons from the same observation to the same salient target and, after each collision, revises its choice in the same way and selects the same target again. The second candidate, a higher temperature, separates the agents on some steps through independent sampling, yet agents sampling from one shared distribution collide with positive probability at every step, and each collision consumes a step, so the separation never becomes reliable. Both candidates therefore exhibit the symmetry trap, since without randomness the agents converge but cannot diverge, and the randomness that separates them on some steps also lowers their Cooperative success (Figure~\ref{fig:divvybench_motivation}, right), where they must converge.

\begin{wrapfigure}{r}{0.65\textwidth}
\centering
\vspace{-12pt}
\includegraphics[width=\linewidth]{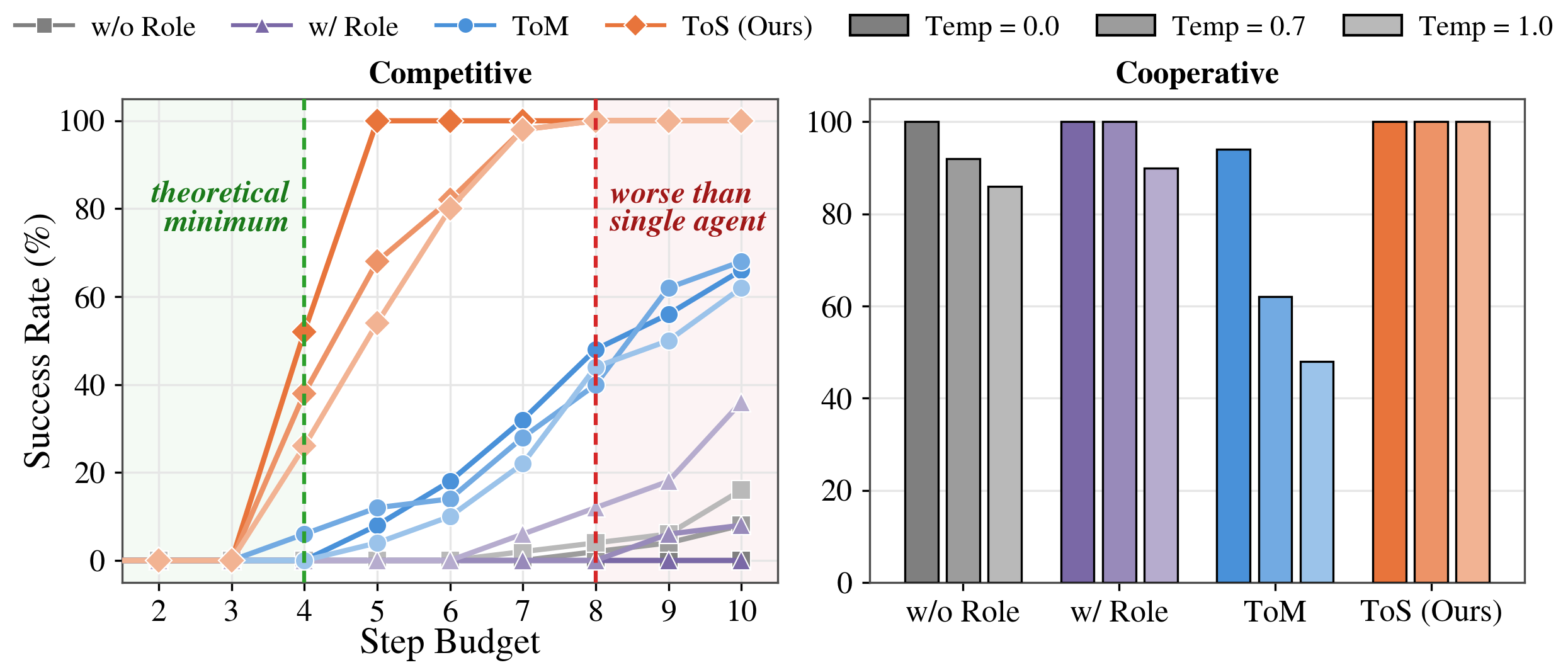}
\vspace{-18pt}
\caption{\textbf{Each baseline succeeds in one setting and fails in the other.} DivvyBench Tabletop with 2 agents, 8 balls, and a 10-step budget. \textit{Left:} Competitive success rate against the step budget. \textit{Right:} Cooperative success rate at the full budget.}
\vspace{-5pt}
\label{fig:divvybench_motivation}
\end{wrapfigure}

\paragraph{ToM Fails under Cooperative Symmetry.}
The third candidate is ToM, in which each agent predicts the others and best-responds to the prediction. ToM raises success in the Competitive setting, since an agent that predicts which targets the others will choose selects a different one, and even an imperfect prediction allows the team to progress as long as the agents separate. The Cooperative setting imposes a stricter requirement, since the agents must independently select exactly the same target. A homogeneous agent that predicts its teammate simulates a copy of itself, so the prediction returns the reasoning the agent already had and uses no external distinction, without which homogeneous agents cannot break this symmetry~\citep{angluin1980local, riedl2026emergent}. Their predictions therefore agree when one target is clearly salient and differ when several are equally salient, since each agent resolves the ambiguity privately through independent decoding, and the agents then select different targets and make no progress, increasingly so as the temperature rises (Figure~\ref{fig:divvybench_motivation}, right). ToM predicts the teammate's target with more than 90\% accuracy in the Cooperative setting, where agents succeed even without prediction, but with only about 30\% accuracy in the Competitive setting, where the agents must divide the targets (Appendix~\ref{appendix:failure_tom}). The prediction is thus accurate where it is least needed and inaccurate where it is needed most, so ToM also exhibits the symmetry trap.

\section{Theory of Scene}
\label{sec:method}

\paragraph{Setup.} We consider $N$ homogeneous agents that coordinate over a horizon of $T$ steps without communication. All agents instantiate the same policy $\pi_\theta$ under a shared reasoning schema $R$, distinguished only by a public role $r_i$. At each step $t=1,\dots,T$, agent $i$ has the environment observation $o_i^t$ (the task context, the current scene, and its legal actions) and the other agents' past actions $\mathbf{a}_{-i}^{<t}$, where $-i$ denotes the agents other than $i$, and produces one action through a single policy call:
\begin{equation}
a_i^t \sim \pi_\theta(\cdot\mid o_i^t,\,r_i,\,\mathbf{a}_{-i}^{<t},\,R)
\quad \text{for all } i=1,\dots,N \text{ concurrently}.
\label{eq:policy}
\end{equation}
Each agent acts without observing the others' actions at the same step, so $\mathbf{a}_{-i}^{<t}$ runs only through step $t-1$. The environment then executes the joint action and returns the next observations.

\subsection{Theory of Mind versus Theory of Scene}
\label{sec:method:compare}

\paragraph{Two Reasoning Schemas.} A coordination method is a choice of reasoning schema $R$, the ordered fields a single $\pi_\theta$ call produces autoregressively, each conditioned on the observation, the role, and the earlier fields it reads. ToM and ToS share the \texttt{state} and \texttt{action} fields and differ in the two middle fields that carry the coordination strategy, and \texttt{action} acts on what those fields produce:
\begin{itemize}
\item \textbf{ToM} (Algorithm~\ref{alg:tom}) corrects a belief about the other agents, predicts their next actions, and best-responds,
\begin{equation}
\underbrace{b_i^{t}\sim\pi_\theta\!\big(\cdot\mid o_i^{t}, r_i, (\mathbf{a}_{-i}^{<t}, \hat{\mathbf{a}}_{-i}^{<t})\big)}_{\texttt{belief}},
\quad
\underbrace{\hat{\mathbf{a}}_{-i}^{t}\sim\pi_\theta\!\big(\cdot\mid o_i^{t}, r_i, b_i^{t}\big)}_{\texttt{predict}},
\quad
\underbrace{a_i^{t}\sim\pi_\theta\!\big(\cdot\mid o_i^{t}, r_i, \hat{\mathbf{a}}_{-i}^{t}\big)}_{\texttt{action}},
\label{eq:tom}
\end{equation}
where $\hat{\mathbf{a}}_{-i}^{t}$ denotes agent $i$'s prediction of the other agents' actions at step $t$ and $\hat{\mathbf{a}}_{-i}^{<t}$ its predictions from the earlier steps, and the belief $b_i^{t}$ is revised by comparing each past action in $\mathbf{a}_{-i}^{<t}$ with the prediction made for it, the correction that defines ToM.
\item \textbf{ToS} (Algorithm~\ref{alg:tos}) reads the present scene and acts on what it reads,
\begin{equation}
\underbrace{\rho_i^{t}\sim\pi_\theta\!\big(\cdot\mid o_i^{t},r_i\big)}_{\texttt{role}},
\quad
\underbrace{\tau_i^{t}\sim\pi_\theta\!\big(\cdot\mid o_i^{t},r_i\big)}_{\texttt{task}},
\quad
\underbrace{a_i^{t}\sim\pi_\theta\!\big(\cdot\mid o_i^{t},r_i,\rho_i^{t},\tau_i^{t}\big)}_{\texttt{action}},
\label{eq:tos}
\end{equation}
where the role reading $\rho_i^{t}$ and task reading $\tau_i^{t}$ are inferred independently from $o_i^{t}$ and $r_i$, and swapping the order in which they are emitted does not significantly lower performance (Table~\ref{tab:ablation}).
\end{itemize}

\begin{figure}[h]
\vspace{-4mm}
\centering
\begin{minipage}[t]{0.49\textwidth}
\begin{algorithm}[H]
\caption{Theory of Mind}
\label{alg:tom}
\small
\begin{algorithmic}[1]
\Require observation $o$,
\Statex \hspace{21pt} \textcolor{tomcolor}{the other agents' past actions $\mathbf{a}_{-i}^{<t}$}
\State \texttt{state}: ground in the observation $o$
\State \textcolor{tomcolor}{\texttt{belief}: correct from predicted vs.\ actual actions}
\State \textcolor{tomcolor}{\texttt{predict}: predict the others' next actions}
\State \texttt{action}: best-respond to \textcolor{tomcolor}{the predicted actions}
\end{algorithmic}
\end{algorithm}
\end{minipage}
\hfill
\begin{minipage}[t]{0.49\textwidth}
\begin{algorithm}[H]
\caption{Theory of Scene (Ours)}
\label{alg:tos}
\small
\begin{algorithmic}[1]
\Require observation $o$
\Statex
\State \texttt{state}: ground in the observation $o$
\State \textcolor{toscolor}{\texttt{role}: infer overlapping vs.\ divided ownership}
\State \textcolor{toscolor}{\texttt{task}: infer converge vs.\ diverge coordination}
\State \texttt{action}: derive \textcolor{toscolor}{a division of labor}
\end{algorithmic}
\end{algorithm}
\end{minipage}
\end{figure}

\begin{remark}[ToS Reasoning without the Other Agents' Actions]
Unlike ToM, whose belief and prediction are read off the other agents' past actions $\mathbf{a}_{-i}^{<t}$, ToS reads only the observation $o_i^{t}$ and its role $r_i$. The observation holds the task context and the global world state, which targets remain, which stations are occupied, and whether the last step took anything, and it attributes nothing to any agent. ToS reads its division from that state, corrects no belief about any agent, and therefore applies where the other agents cannot be observed. We hand both schemas the same inputs so that $R$ is the only variable in the comparison, and withholding $\mathbf{a}_{-i}^{<t}$ does not significantly lower ToS (Table~\ref{tab:ablation}).
\end{remark}

\subsection{Role Gating and Task Coupling}
\label{sec:method:tos}

\paragraph{Role Gating $\rho_i^{t}$.} Role gating turns the public role $r_i$ into a reading of the scene and takes two inputs. First, the observation $o_i^{t}$ carries the task context and the current scene together with the outcome of the agent's own actions. Second, the public role $r_i$ is the one signal that tells homogeneous agents apart, fixed before the episode and unchanged as the scene changes. Whether it is a specific assignment or only a name is a joint property of the role and the observation, so one label divides the work in one scene and divides none of it in another. The \texttt{role} field combines them into the role reading $\rho_i^{t}$, and the gate settles only whether a division has to be derived at all, while the task reading determines the division itself. The reading classifies the ownership of the work:
\begin{itemize}
\item \emph{Overlapping}, where the scene and the role name no agent's part and the agents work the same targets, the role distinguishing the agents without reserving work for any of them. The agent derives a division from the task reading and takes the part its role indexes, since agents that reason alike would otherwise fall into the symmetry trap.
\item \emph{Divided}, where the scene and the role already name each agent's part. The agent advances its own part, taking the single most useful action toward the goal and first producing any input a later step still lacks.
\end{itemize}

\paragraph{Task Coupling $\tau_i^{t}$.} Task coupling turns each target into a division that homogeneous agents derive the same way. The \texttt{task} field infers what each target requires of the team and records it as the task reading $\tau_i^{t}$. The schema names three couplings:
\begin{itemize}
\item \emph{Joint}, where the target's outcome requires more than one agent on it at once, which the team answers by converging on it.
\item \emph{Exclusive}, where the target is rival and what one agent takes is taken from the rest, which the team answers by dividing it into disjoint shares, each indexed by an agent's role $r_i$.
\item \emph{Sequential}, where the target depends on an input not yet in place, which the team answers by taking its stages in turn, each agent working the stage its role owns.
\end{itemize}

\paragraph{Combining $\rho_i^{t}$ and $\tau_i^{t}$ into the Action $a_i^{t}$.}
The last statement of Equation~\eqref{eq:tos} conditions the action on both readings, and in the \texttt{action} field the agent derives from them the division of labor for the whole team and acts on it at every step. A scene can admit more than one rational division, since its targets can be ordered by their position in the observation, by name, by difficulty, distance, or priority, or by any other property they carry. Homogeneous agents reason alike and therefore arrive at the same division more readily, whereas agents with different reasoning styles can derive different ones. Across the steps of an agentic interaction, the scene changes at every step and realigns those divisions until the agents agree on one. The one difference between the agents (i.e., the public role) then fixes which part of that division each takes, so every agent acts on its own part without exchanging a message.

\subsection{Analysis of the Coordination Probability}
\label{sec:method:analysis}

Consider a single step, in which a scene offers $K\ge N$ targets indexed by $k$ and every agent selects one of them. The policy $\pi_\theta$ of Equation~\eqref{eq:policy} induces for agent $i$ a distribution $p^{(i)}$ over the targets, with $p^{(i)}_k$ the probability of target $k$, and the $N$ agents sample from their distributions independently, since no agent sees another's action at the same step. Two probabilities summarize the step, the probability $P_{\mathrm{div}}$ that the $N$ agents select $N$ distinct targets, an event termed divergence, and the probability $P_{\mathrm{conv}}$ that all of them select the same target, termed convergence. A schema blind to $r_i$ puts the agents on one distribution whenever their remaining inputs coincide, a condition the first step of an episode satisfies. Two propositions bound these probabilities, one for that shared distribution and one for the general case, and Appendix~\ref{appendix:proofs} proves both.

\begin{proposition}[Bounds under One Shared Distribution]\label{prop:blind}
If all $N$ agents sample independently from one distribution $p$, a single step satisfies
\begin{equation}
\label{eq:step}
\begin{aligned}
&\textstyle P_{\mathrm{div}}(p) = N!\,e_N(p), \quad
&&\textstyle\max_{p}\,P_{\mathrm{div}}(p) = \prod_{n=1}^{N-1}\Big(1-\tfrac{n}{K}\Big)<1,\\
&\textstyle P_{\mathrm{conv}}(p) = \sum_{k=1}^{K}p_k^{N}, \quad
&&\textstyle\max_{p}\,P_{\mathrm{conv}}(p) = 1,
\end{aligned}
\end{equation}
where $e_N(p)=\sum_{k_1<\cdots<k_N}p_{k_1}\cdots p_{k_N}$ is the elementary symmetric polynomial of degree $N$.
\end{proposition}

\begin{remark}[Policy Independence]
Proposition~\ref{prop:blind} bounds every schema whose agents sample one distribution, which a schema blind to $r_i$ does whenever their inputs coincide. The bounds of Equation~\eqref{eq:step} depend on $N$ and $K$ alone, so they hold for every model backbone, every shared prompt, and every temperature at once, while a temperature sweep reaches finitely many policies. At $N=2$ and $K=8$, as in our DivvyBench experiments, converging is unconstrained, while diverging succeeds with probability at most $1-1/8=7/8$ in one step, since agents sampling one shared distribution select the same target with probability at least $1/8$. Compounded over the four steps of the theoretical minimum, the bound is $0.27$, even when the distribution is retuned as targets are taken.
\end{remark}

A schema that reads $r_i$ can give different agents different distributions, and we quantify the effect of that difference. Let $\delta_{ij}\in[0,1]$ denote the total variation distance between the distributions of agents $i$ and $j$. It satisfies $\delta_{ij}=0$ when the two distributions are identical and $\delta_{ij}=1$ when they share no target, and its complement $1-\delta_{ij}=\sum_k\min\big(p^{(i)}_k,p^{(j)}_k\big)$ is the overlap of the two distributions. Let $\delta_{\min}$ and $\delta_{\max}$ be the smallest and largest $\delta_{ij}$ over the pairs of agents.

\begin{proposition}[Bounds under Per-Agent Distributions]\label{prop:role}
If the $N$ agents sample independently, agent $i$ from its distribution $p^{(i)}$, then a single step satisfies
\begin{equation}
\label{eq:role}
\begin{aligned}
&P_{\mathrm{div}} \le 1-\tfrac{1}{K}\big(1-\delta_{\min}\big)^{2}, \quad
&&P_{\mathrm{div}} = 1 \text{ only if } \delta_{\min}=1,\\
&P_{\mathrm{conv}} \le 1-\delta_{\max}, \quad
&&P_{\mathrm{conv}} = 1 \text{ only if } \delta_{\max}=0.
\end{aligned}
\end{equation}
\end{proposition}

\begin{remark}[The Two Readings of ToS Adjust $\delta_{ij}$]
The distances enter the two bounds of Equation~\eqref{eq:role} with opposite signs. Raising $\delta_{\min}$ raises the bound on $P_{\mathrm{div}}$, since distributions that overlap less collide less, and raising $\delta_{\max}$ lowers the bound on $P_{\mathrm{conv}}$, since a pair selects the same target with probability at most its overlap $1-\delta_{ij}$. The two only-if conditions of Equation~\eqref{eq:role} therefore pull $\delta_{ij}$ to opposite ends. The demand is per target, since the pair selects target $k$ together with probability $p^{(i)}_k p^{(j)}_k$. Targets the team must divide require $p^{(i)}_k p^{(j)}_k=0$ and thus $\delta_{ij}=1$, and a target it must share requires the reverse, so the same pair needs $\delta_{ij}$ to change from one step to the next with the targets that step addresses. The two readings of Equation~\eqref{eq:tos} place each target in its group, with the task reading $\tau_i^{t}$ inferring from the shared scene which demand a target makes and the role reading $\rho_i^{t}$ settling whether the distributions differ at all. Appendix~\ref{appendix:delta} measures $\delta_{ij}$ under each method.
\end{remark}

\section{Experiments}
\label{sec:experiments}

\subsection{Experimental Setup}
\label{sec:exp:setup}

\paragraph{Environments and Metrics.} We evaluate on three agentic benchmarks, all run with the agents acting concurrently and without any inter-agent communication, and Appendix~\ref{appendix:setup} gives their rules and per-benchmark parameters in full.
\begin{itemize}
\item \textbf{DivvyBench:} 2 agents must take all 8 targets within 10 steps under the Competitive, Cooperative, and Mixed settings. We report success and stuck rate, the fraction of steps that take nothing, over 150 episodes per setting, 50 in each of the Tabletop, Airspace, and Household scenarios.
\item \textbf{GovSim}~\citep{piatti2024cooperate}: 5 agents privately harvest from one regenerating resource pool for at most 12 rounds, and the pool collapses once it falls below 5 units. We report survival and total gain over 50 games in each of the Fishery, Pasture, and Pollution scenarios.
\item \textbf{Overcooked}~\citep{carroll2019utility, sun2025collab}: 2 agents cook over 6 difficulty levels within a 50-step episode, sharing every single-use station. Success saturates, so we report throughput in dishes per episode over 15 episodes per level, alongside an average normalized to ReAct.
\end{itemize}

\paragraph{Baselines.} We group the baselines by the signal each uses to coordinate. \textbf{Prediction-free} baselines carry no model of the other agents, and comprise ReAct~\citep{yao2023react}, which reasons over the shared observation and takes the single most useful action, Focal~\citep{aharon2026tacit}, which applies the focal-point prompt unchanged and takes the option that stands out from the rest, and CFD~\citep{hayler2026zero}, which adds a self-check that rejects any action resting on an arbitrary, flippable convention. \textbf{ToM} baselines predict the other agents, and comprise Classic ToM (Algorithm~\ref{alg:tom}), which best-responds to its prediction of their next actions, ProAgent~\citep{zhang2024proagent}, which infers the intention behind their actions and plans a skill that complements it, and HypMinds~\citep{cross2025hypothetical}, which maintains and tests hypotheses about their hidden strategies.

\paragraph{Implementation.} All agents run one LLM backbone, Qwen3.5-35B-A3B~\citep{qwen3.5}, under a prompt that is identical at each step except for the single line naming the public role (e.g., Chef or Assistant in Overcooked, Leader or Follower in DivvyBench, and A0 to A4 in GovSim). Across methods only the reasoning schema changes, itself appended verbatim to every task prompt over the three benchmarks, and we decode at a fixed temperature of $0.7$. Each cell reports the mean over the trials above, with the spread ($\pm$) the standard error of that mean. Appendix~\ref{appendix:config} gives the rest of the decoding and serving configuration, and Appendix~\ref{appendix:prompt} the task prompt of each DivvyBench setting with the ToS schema appended.

\subsection{Results}
\label{sec:exp:main}

\begin{table}[t]
\centering
\caption{\textbf{ToS is the only method that succeeds on all three benchmarks, and it outperforms every baseline on almost every metric.} Best per column in bold, ToS shaded, and $\uparrow$ and $\downarrow$ mark the improving direction. Appendix~\ref{appendix:significance} reports a two-sided 95\% confidence interval and a significance test for every gap on the headline statistic of each benchmark.}
\vspace{-5pt}
\label{tab:main}
\resizebox{\columnwidth}{!}{\input{00_results/table_main.tex}}
\end{table}

\paragraph{Comparison across the Coordination Spectrum.}
Each baseline in Table~\ref{tab:main} collapses for one of three reasons. \emph{(i) Nothing to break the symmetry.} ReAct acts on the shared observation alone, so homogeneous agents commit to the same option. They select the same target on DivvyBench and the same station on Overcooked, where no agent prepares the input that station needs, and with five agents on a shared pool each harvests with no reason to hold back. \emph{(ii) A prediction that mirrors the predictor.} The ToM methods are the three strongest baselines in DivvyBench Competitive, ReAct alone exceeds all three in DivvyBench Cooperative, and none survives half the GovSim horizon on average. Predicting a homogeneous agent returns the predictor's own ambiguity among equally salient targets, which splits agents that must share a target. GovSim rewards each agent for what it takes from a resource the others harvest as well, so the prediction that they will harvest becomes a reason to harvest first, and the pool collapses. \emph{(iii) One rule for every situation.} Focal separates even where the task requires the agents to meet, because a salience criterion treats every target the same way, and CFD never commits, because it rejects any action resting on a convention that could have been flipped. Neither reads what each target requires of the team, so each fails wherever its fixed rule does not fit the target. ToS avoids all three failure modes by deriving the division of labor from the scene at each step, which breaks the symmetry without predicting the other agents.

\begin{figure}[t]
\centering
\includegraphics[width=\linewidth]{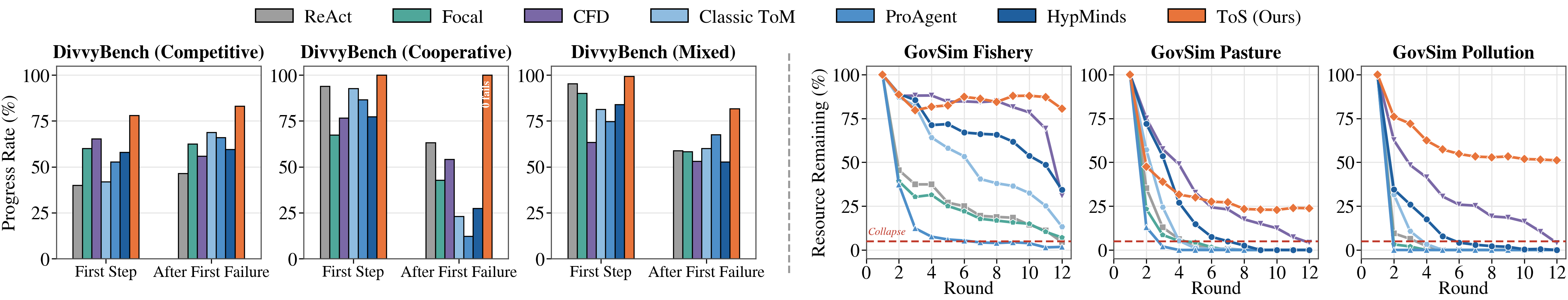}
\vspace{-12pt}
\caption{\textbf{ToS breaks symmetry on the first step, recovers after a failure, and sustains the resource pool.} \textit{Left:} On DivvyBench, the progress rate at the first step and at the step after the first failure, in each setting. The progress rate is defined as the fraction of episodes in which the team takes at least one target on that step. \textit{Right:} On GovSim, the resource remaining over the horizon.}
\vspace{-12pt}
\label{fig:recovery}
\end{figure}

\paragraph{First Step vs. Recovery after Failure.}
Figure~\ref{fig:recovery} reports the first step and the step after the first failure. At the first step no action has been taken, so every method holds the same observation and public role, and the reasoning schema is the only source of the lead ToS holds in every setting. No baseline uses the one asymmetry it is given, since ReAct acts on the observation, the ToM methods act on a prediction formed without seeing the other agents act, and Focal and CFD apply a rule fixed before the scene is read. At the step after the first failure, the methods differ only in what each infers from it. ToS leads at both steps, since a division derived from a scene that every agent reads alike fails only when an agent misreads the scene, whereas independent selections fail whenever they coincide. Over the 12 rounds of GovSim, the pool declines under every baseline, whereas ToS stabilizes it within a few rounds and ends each scenario with more resource than any baseline.

\begin{table}[t]
\centering
\caption{\textbf{Ablation Study and Robustness Study.} Removing role gating or task coupling lowers performance, while reordering the two readings or withholding the other agents' actions does not. Red shaded cells are significantly worse than ToS (one-sided two-sample $z$-test, $p<0.05$).}
\label{tab:ablation}
\vspace{-5pt}
\resizebox{\columnwidth}{!}{\input{00_results/table_ablation.tex}}
\vspace{-12pt}
\end{table}

\paragraph{Ablation Study.}
Table~\ref{tab:ablation} removes one reading at a time. Under \emph{w/o \texttt{task}}, the agents cannot settle whether to converge on the same target or diverge onto different ones, so all three DivvyBench settings degrade, and GovSim loses the reading that marks the pool as rival, so nothing tells an agent that what it takes is taken from the rest. Overcooked alone barely changes, since its recipe states the ordering outright. Under \emph{w/o \texttt{role}}, the pattern reverses and Overcooked drops, since it loses the reading that reports the work already divided, and the agents contend over stations the roles already assign. GovSim drops again by the wider margin, since role gating no longer reports the pool as contended, and the agents divide its full capacity, which exceeds the harvest the pool can sustain. By contrast, DivvyBench alone is unaffected, since the task coupling together with the public role (i.e., Leader or Follower) is enough to derive the division of labor. The task context and the public role settle a different part of the division on each benchmark, so either removal costs two of the three benchmarks, and only both readings together hold on all three.

\paragraph{Robustness Study.}
Table~\ref{tab:ablation} also reorders the two readings and withholds the other agents' actions, and neither perturbation produces a significant drop. Reordering them under \emph{\texttt{task}$\rightarrow$\texttt{role}} does not change the inputs of either reading, since Equation~\eqref{eq:tos} infers both from the observation $o_i^{t}$ and role $r_i$ alone. Withholding the other agents' actions under \emph{w/o $\mathbf{a}_{-i}^{<t}$} is the stronger test, since every ToM baseline depends on this input, and ToS therefore applies wherever those actions are unavailable.

\begin{wraptable}{r}{0.5\textwidth}
\centering
\vspace{-12pt}
\caption{\textbf{Classic ToM with ToS's \texttt{role} and \texttt{task} readings.}}
\label{tab:tom_tos}
\vspace{-5pt}
\resizebox{0.5\textwidth}{!}{\input{00_results/table_tom_tos.tex}}
\vspace{2pt}
\end{wraptable}

\paragraph{ToM + ToS's Readings.}
Table~\ref{tab:tom_tos} provides Classic ToM with the reading text of ToS, leaving the mechanism that derives the action as the only difference between the two schemas. ToS conditions the action on both readings (Algorithm~\ref{alg:tos}, Line 4), whereas each variant best-responds to its prediction (Algorithm~\ref{alg:tom}, Line 4), so that a reading reaches the action only through that prediction, with two consequences. \emph{(i) ToM with both readings still fails to converge.} Adding both readings raises GovSim survival only to 39.1 against 85.2 for ToS and lowers DivvyBench Cooperative success below that of Classic ToM alone, the two cases in which every agent must select the same action. A prediction is specific to each agent and revised from its own history, so actions derived from predictions coincide only when the predictions themselves do, whereas the readings are common to all agents. The \texttt{role} field alone lowers Classic ToM on DivvyBench and GovSim, since it specifies no action on which a best response can operate. \emph{(ii) Canonical order does not determine the action.} The \texttt{task} field provided to Classic ToM carries the same instruction as in ToS, to apply the rule of each coupling over a canonical order, so the variants derive that order in the same manner as ToS. Agents that derive the same order therefore select the same action under ToS, whereas under ToM they act on predictions that the order does not constrain and that need not coincide with it.

\paragraph{Further Analysis.}
Appendix~\ref{appendix:further_analysis} covers the public role (\ref{appendix:1bit}), observation asymmetry (\ref{appendix:observation_asymmetry}), the model backbone (\ref{appendix:impact_backbone}), heterogeneous models (\ref{appendix:heterogeneous}), heterogeneous reasoning schemas (\ref{appendix:het_reasoning}), the number of agents (\ref{appendix:impact_number_agents}), the interaction horizon (\ref{appendix:impact_horizon}), ToM's failure mode (\ref{appendix:failure_tom}), per-agent distributions (\ref{appendix:delta}), and confidence intervals and significance tests (\ref{appendix:significance}).

\section{Conclusion}

In this work, we identify the symmetry trap, in which homogeneous LLM agents acting concurrently without communication fail both to divide contested targets and to converge on shared ones. Existing approaches each handle only one side of the trap, since a public role or sampling randomness keeps the agents together where they must divide, whereas predicting one another separates them where they must converge. Theory of Scene (ToS) escapes the trap and remains effective in every setting of DivvyBench, GovSim, and Overcooked, without a predefined division of labor or training on the task. These results suggest that LLM agents coordinating without communication are better served by reasoning from the scene they share than by predicting one another.

\newpage

\bibliography{references}
\bibliographystyle{iclr2027_conference}

\newpage
\appendix
\begin{center}{\bf\Large Appendix}\end{center}
\startcontents[sections]
\printcontents[sections]{l}{1}{\setcounter{tocdepth}{3}}

\newpage
\section{Further Analysis}
\label{appendix:further_analysis}

\subsection{Impact of the Public Role}
\label{appendix:1bit}

\paragraph{Public Role.}
Table~\ref{tab:1bit} holds the ToS reasoning fixed and changes only the public role $r_i$ the agents are given, from an identical description through an index to the leader and follower labels ToS uses with two agents. At temperature $0$, where the wording of $r_i$ is the only thing that can separate the agents, identical descriptions leave ToS at 0.7\% in DivvyBench Competitive, the outcome classical symmetry breaking predicts for identical deterministic agents~\citep{angluin1980local}, so the trap survives an arbitrarily good reasoning schema once nothing distinguishes the agents. DivvyBench Cooperative is untouched at 100\% with no stuck steps, since converging on one target requires only agreement, and the distinction is therefore needed exactly where the agents must differ. An index already recovers most of the success in DivvyBench Competitive (92.7\%), and the leader and follower labels carry the same 1 bit yet close the remainder and cut the stuck rate in DivvyBench Mixed from 9.8\% to 6.4\%, so equal information does not yield equal coordination. Beyond two agents, every method uses index labels, and Appendix~\ref{appendix:impact_number_agents} still finds ToS near perfect at every team size.

\paragraph{Sampling Temperature.}
Table~\ref{tab:1bit} also sweeps the sampling temperature against each public role, the other candidate for breaking the symmetry~\citep{patil2026randomness}. Temperature leaves the identical description near the floor, at 11.3\% in DivvyBench Competitive at the highest temperature, partly compensates for a weak distinction, carrying the index from 92.7\% to 99.3\%, and adds nothing once the bit is a leader and follower label, where temperature $0$ already completes all three DivvyBench settings and gives the lowest stuck rate of all. Randomness therefore substitutes for the wording of a distinction and never for the distinction itself.

\begin{table}[H]
\centering
\caption{\textbf{One bit of distinction suffices to break the symmetry.} The ToS reasoning is held fixed on DivvyBench and only the public role differs. Under None the agents receive the same wording, under 1 Bit (index) they are told apart by ``A0'' and ``A1'', and under 1 Bit (role) by the leader and follower labels ToS uses, the last two carrying the same 1 bit. Red marks a cell significantly worse than 1 Bit (role) in the same setting at the same temperature (one-sided two-sample $z$-test, $p<0.05$).}
\label{tab:1bit}
\resizebox{\columnwidth}{!}{\input{00_results/table_1bit.tex}}
\end{table}

\subsection{Impact of Observation Asymmetry}
\label{appendix:observation_asymmetry}

\paragraph{Observation Asymmetry.}
Table~\ref{tab:impact_observation_asymmetry} holds the scene fixed and changes only the order in which each agent is shown it. An observation that differs is a third source of asymmetry between homogeneous agents, after the public role of Appendix~\ref{appendix:1bit} and the temperature of Section~\ref{sec:background}, and like the temperature it raises divergence and lowers convergence. Each baseline selects the most salient target of its own listing, and those listings no longer agree, so DivvyBench Competitive success rises. Converging instead requires the agents to select one target, and each selects the head of its own listing, which under a shared order had been the same target for all, so every baseline collapses in DivvyBench Cooperative. ReAct marks both extremes, rising in DivvyBench Competitive, where the agents must divide, and collapsing where they must meet, so the method that gains most from the asymmetry in one setting loses most in the other. DivvyBench Mixed puts both demands in one episode and leaves every baseline collapsed, so an asymmetry in what the agents observe resolves one side of the trap and never both.

\paragraph{ToS Derives Canonical Orders.}
Table~\ref{tab:impact_observation_asymmetry} also shows that ToS is the only method that succeeds in all three settings. A scene admits one or more canonical orders over its targets, from their names, from an index, or from any other property they carry, and ToS derives one and divides it by role. One derived order serves both demands, since converging takes its first element and diverging splits it among the roles, and the agents therefore reach the same division from listings that agree nowhere. Two candidate orders are available, the one a listing presents and the one the names induce, and only Tabletop separates them, since its colors are listed out of alphabetical order while the grid references and the location names are already sorted. Given one shared listing, the agents follow the order it presents in 93.8\% of their decisions, and under a reversed or shuffled listing they sort the names instead, in 61.7\% and 65.2\% of their decisions. Every agent can construct a sorted order for itself, whereas the presented order is no longer shared once the perturbation withholds it. The disagreements that remain fall mostly on the first step, where the agents hold the scene and nothing else, and the steps that follow realign them, since a wasted step is visible in the observation and the next action is sampled independently of the last.

\begin{table}[H]
\centering
\caption{\textbf{Only ToS succeeds in every setting under disagreeing observations.} Under Shared, the agents are shown the targets in the presentation order, as in Table~\ref{tab:main}. Under Reversed and Shuffled, A1 is shown them in reverse or in a fixed derangement of that order, with the targets and their types unchanged, and the rules state that the orders may differ. Red marks a setting whose Success is 20\% or below.}
\label{tab:impact_observation_asymmetry}
\resizebox{\columnwidth}{!}{\input{00_results/impact_observation_asymmetry.tex}}
\end{table}

\subsection{Impact of the Model Backbone}
\label{appendix:impact_backbone}

Table~\ref{tab:impact_backbone} rebuilds the agents on three further backbones, spanning three families from 9B to 26B, and repeats the three DivvyBench settings on each in every scenario. ToS is near-perfect on all four and holds the highest average on each, while baseline success swings sharply from one backbone to the next. The cause is that no baseline breaks symmetry explicitly, so whether homogeneous agents converge or diverge is set by the backbone, whatever the task requires. The prediction-free ReAct and CFD read that decision off the backbone's raw decoding tendency, and the prediction-based Classic ToM, ProAgent, and HypMinds reach the same outcome through prediction, since predicting an agent that reasons alike returns the predictor's own tendency. The ranking among the baselines changes with the backbone, and HypMinds ranks first on Gemma-4-26B and fifth on GPT-OSS-20B. ToS derives its decision from the scene, an input that does not vary with the backbone.

\begin{table}[H]
\centering
\caption{\textbf{ToS's advantage holds across model backbones.} Three further backbones (Gemma-4-26B, GPT-OSS-20B, and Qwen3.5-9B) stand in for the main Qwen3.5-35B-A3B across the three DivvyBench settings, with all agents in a run sharing one backbone.}
\label{tab:impact_backbone}
\resizebox{\columnwidth}{!}{\input{00_results/impact_backbone.tex}}
\end{table}

\subsection{Impact of Heterogeneous Models}
\label{appendix:heterogeneous}

Table~\ref{tab:impact_heterogeneous} gives the other agent a different backbone, the different-family Gemma-4-26B and the weaker same-family Qwen3.5-9B. Heterogeneity is the condition Theory of Mind is built for, since a prediction of an agent that reasons and decodes alike returns the predictor's own action, and a different backbone removes that. Every baseline gains against Gemma-4-26B, most of all CFD, which under homogeneity rejects every flippable convention until a second backbone breaks the deadlock for it. Classic ToM becomes the strongest baseline in both pairings, and it still trails ToS by 16.2 points against Gemma-4-26B and by 32.0 against Qwen3.5-9B. The weaker backbone separates the two mechanisms, since it breaks the symmetry just as freely and yet no baseline rises significantly above its homogeneous score, so a different agent improves ToM only when that agent is also a capable one. ToS's success is unchanged in both pairings, because the convention it commits to is read from the shared scene and the other agent contributes only its ability to read that scene and apply the same rule. Table~\ref{tab:crossplay} runs the same check on the reasoning schema, pairing each method with an agent that runs a different one on the same backbone.

\begin{table}[H]
\centering
\caption{\textbf{ToS's advantage holds when the agents run different model backbones.} The agents run the method named in each row on DivvyBench, A0 on the main Qwen3.5-35B-A3B backbone and A1 on the backbone named in each panel.}
\label{tab:impact_heterogeneous}
\resizebox{\columnwidth}{!}{\input{00_results/impact_heterogeneous.tex}}
\end{table}

\subsection{Impact of Heterogeneous Reasoning Schemas}
\label{appendix:het_reasoning}

Table~\ref{tab:crossplay} pairs each method with an agent running ReAct or Classic ToM. Every method coordinates with an agent that reasons differently, and every baseline scores above its homogeneous result in Table~\ref{tab:main} beside an unlike agent, least for Classic ToM. The gain comes from the pairing, since no method was changed and each was paired with the same two agents. Agents running unlike schemas are no longer homogeneous, so in DivvyBench Competitive the pairing supplies the distinction each baseline failed to derive. DivvyBench Cooperative reverses only for ReAct, the one baseline that converges reliably against a copy of itself, and its score there falls beside an unlike schema. The others converge less reliably under homogeneity, so an unlike partner does not lower their Cooperative success. Even with the symmetry broken for them, no baseline reaches ToS beside the same agent, and none reaches homogeneous ToS.

\begin{table}[H]
\centering
\caption{\textbf{ToS achieves the highest success in every setting beside an agent that reasons differently.} A0 runs the method named in each row, and A1 runs the reasoning schema named in each panel.}
\label{tab:crossplay}
\resizebox{\columnwidth}{!}{\input{00_results/crossplay.tex}}
\end{table}

\subsection{Impact of the Number of Agents}
\label{appendix:impact_number_agents}

\paragraph{Scaling the DivvyBench Team.} Figure~\ref{fig:impact_number_agents} grows the DivvyBench team from 2 agents to 5. DivvyBench Competitive absorbs the extra agents, since a collision wastes the step only for the agents in it while the rest still take targets, and the baselines show no common trend there. DivvyBench Cooperative and Mixed run the other way, since a Cooperative target is taken only when all $N$ agents select it in the same step and one deviating agent voids the round. If each agent independently selects the shared target with probability $p<1$, the round succeeds with probability about $p^{N}$, which decays as the group grows however close $p$ is to one. Every baseline declines as the team grows, whatever its schema, since none of them makes the shared target certain. ToS reads the division from the shared scene, an input every agent holds, so $p$ remains near one at every team size, and ToS stays near perfect in all 12 panels.

\begin{figure}[H]
\centering
\includegraphics[width=\linewidth]{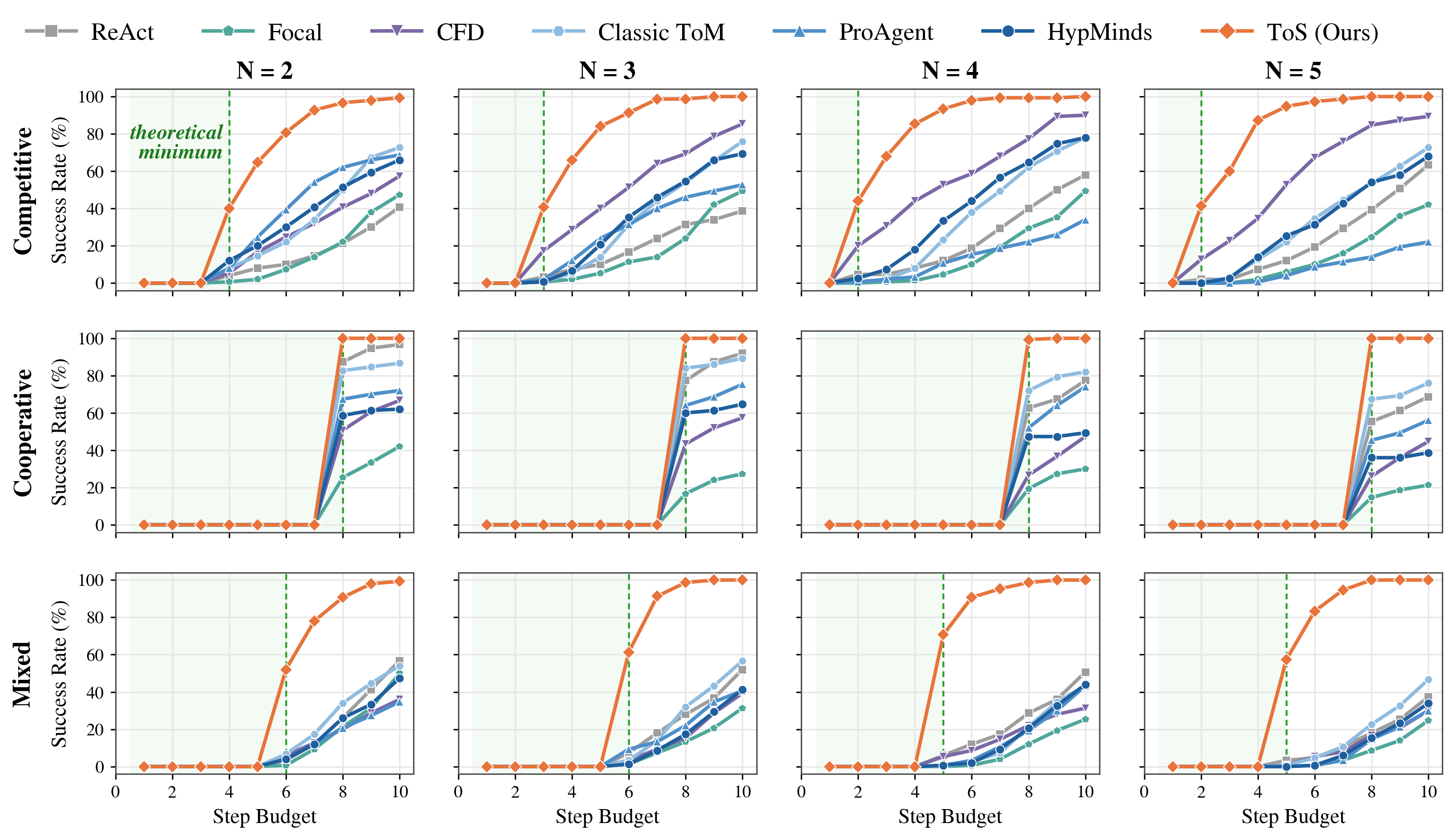}
\caption{\textbf{Team size leaves ToS unchanged and lowers the baselines in DivvyBench Cooperative and Mixed.} DivvyBench success rate versus step budget from 2 to 5 agents, one panel per setting and team size.}
\label{fig:impact_number_agents}
\end{figure}

\paragraph{Scaling the Population.} Figure~\ref{fig:impact_population} scales the GovSim population from 5 to 100 while the pool stays at capacity 100. Every baseline decays as agents are added and has all but collapsed by 50 agents. ToS keeps its survival constant up to 20 agents in Fishery and Pasture and sustains the commons for more than half the horizon at 50. A fixed pool yields a fixed sustainable harvest, so the share each agent may take falls with every agent added, and a single agent that takes the whole pool collapses it for the group. At 100 agents the sustainable share is half a unit and a harvest is an integer, so the pool survives only if half of the agents harvest one unit and the other half harvest nothing.

\begin{figure}[H]
\centering
\includegraphics[width=\linewidth]{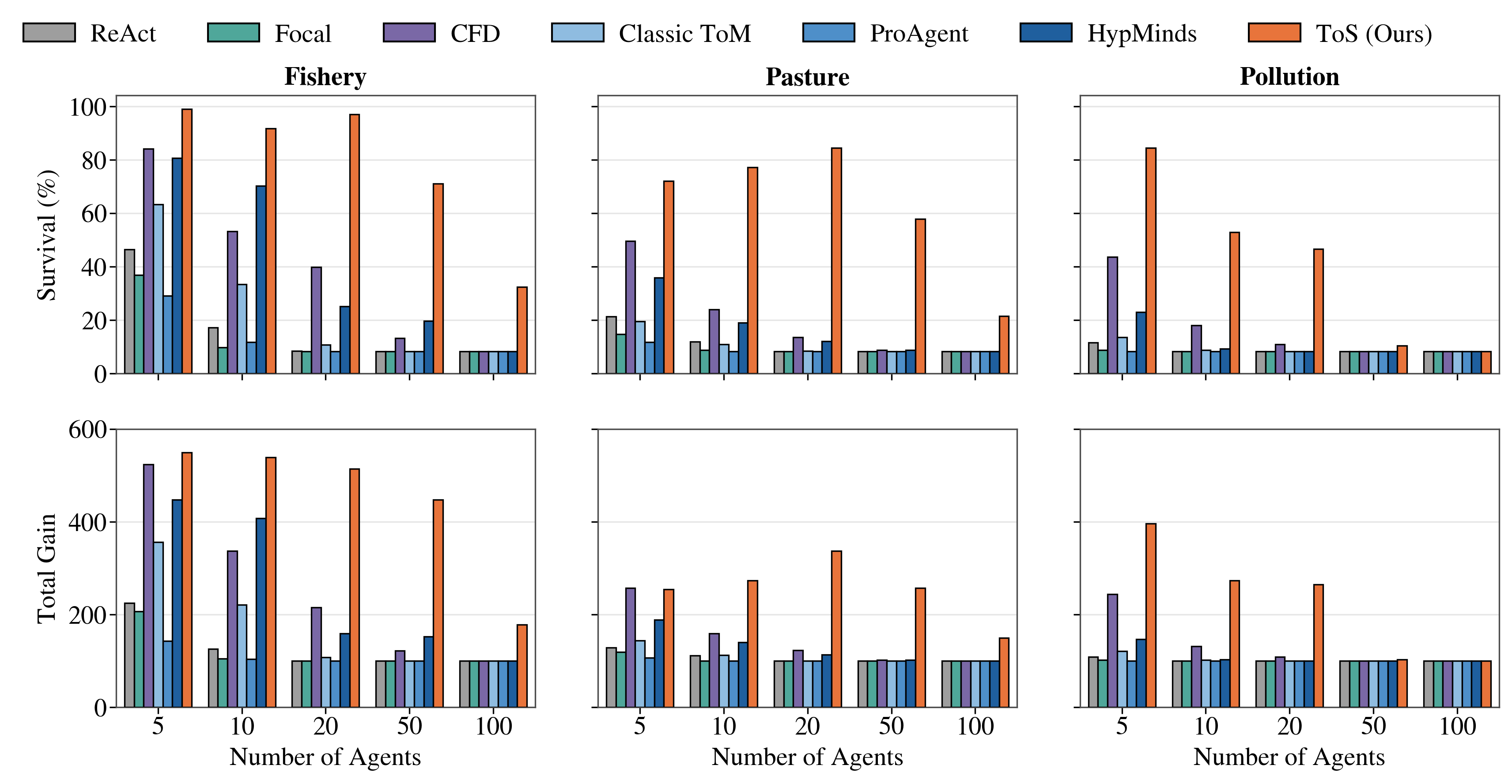}
\caption{\textbf{ToS degrades more slowly than the baselines as the population grows.} GovSim survived rounds and total gain over populations from 5 to 100 on a pool of fixed capacity.}
\label{fig:impact_population}
\end{figure}

\paragraph{Per-Agent Demand Distribution.} Figure~\ref{fig:impact_population_hist} isolates the first round, when the pool is full and every method is still alive, and plots each agent's demand against the sustainable per-agent share. Every baseline leaves over-takers at every population size, and at 100 agents ReAct, Focal, Classic ToM, and ProAgent each leave more than 20\% of the population over-taking, through one of two errors. The prediction-free ReAct and Focal size their demand so that what they leave behind regrows to full capacity, a calculation valid for one harvester, and arrive at half the pool, a quantity that does not fall as agents are added. The prediction-based Classic ToM and ProAgent instead anticipate heavy grazing by the others and harvest ahead of it. CFD and HypMinds make neither error and reduce their demand as the group grows, which is why they leave the fewest over-takers of the baselines. ToS divides the sustainable harvest by the number of agents it reads in the scene, and its over-takers are negligible at every population size. Appendix~\ref{appendix:failure_tom} examines how ToM's prediction itself breaks down as the team grows.

\begin{figure}[H]
\centering
\includegraphics[width=\linewidth]{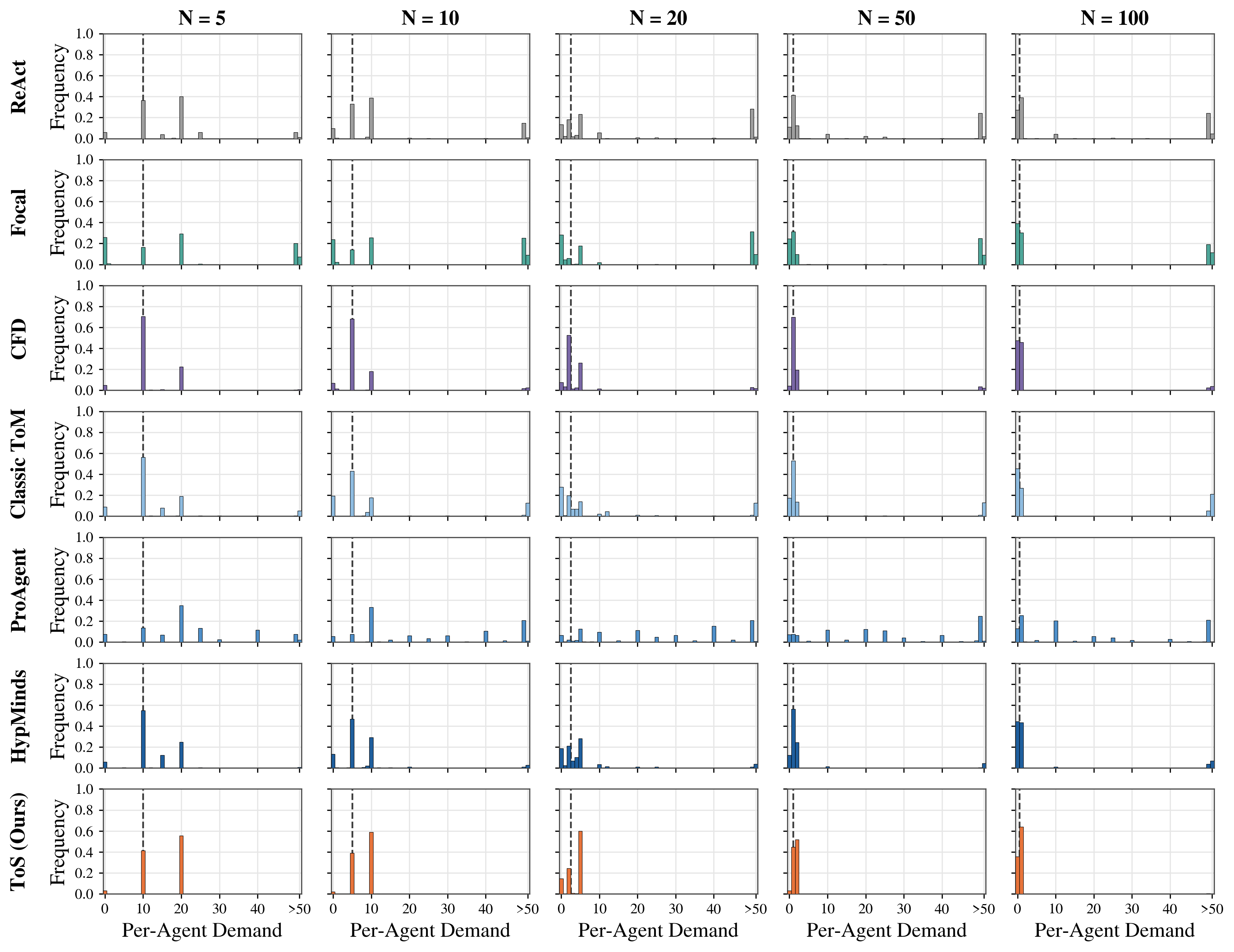}
\caption{\textbf{The baselines leave over-takers at every population size and ToS leaves almost none.} Per-agent demand on GovSim (Pasture, first round), the fraction of the population at each demand level for every method and population, with the dashed line at the sustainable per-agent share. An over-taker asks for half the pool or more, which is the whole team's sustainable allowance and $N$ times the dashed line.}
\label{fig:impact_population_hist}
\end{figure}

\subsection{Impact of the Interaction Horizon}
\label{appendix:impact_horizon}

\paragraph{Long-Horizon Survival.} Figure~\ref{fig:impact_horizon} runs GovSim for 120 rounds, $10\times$ the default horizon. In Fishery, ToS sustains its harvest for the length of the game, and since the horizon is fixed and known, exhausting the resource in the final rounds, when preserving it has no remaining value, is the optimal end-game action. ToS reaches that point with the pool still at capacity, whereas CFD and HypMinds lose theirs earlier, so the fall in the survival curve marks the end game for ToS and a collapse for CFD and HypMinds. ToS also harvests the most over the rounds in which all three are alive, since a pool held at capacity regrows the most each round. The advantage measured over 12 rounds in the main results is therefore a lower bound on what a longer game reveals.

\begin{figure}[H]
\centering
\includegraphics[width=\linewidth]{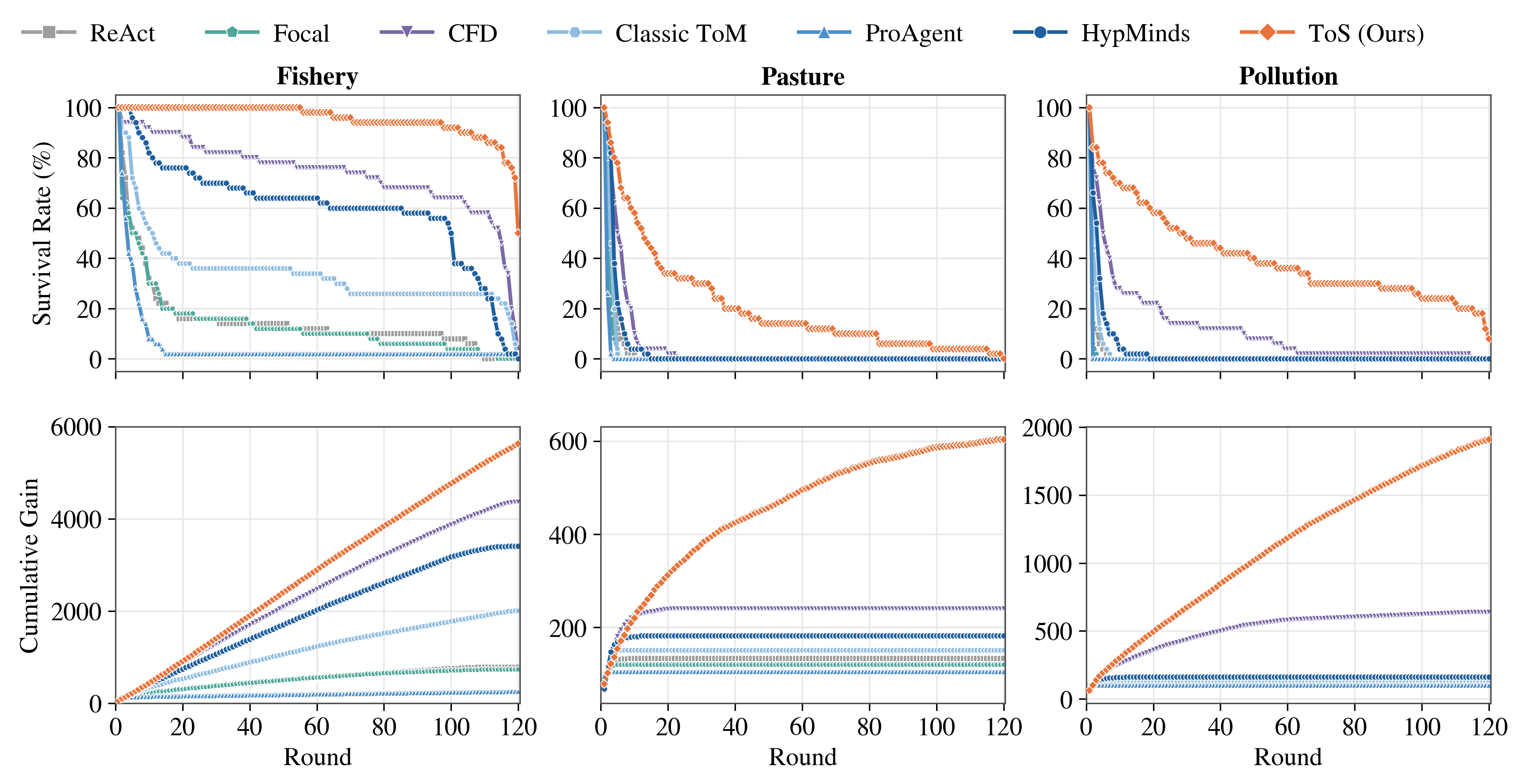}
\caption{\textbf{ToS's advantage compounds over a long horizon.} GovSim run for 120 rounds, $10\times$ the default horizon, reporting the survival rate (the fraction of runs whose commons is still alive at each round) and the cumulative total gain.}
\label{fig:impact_horizon}
\end{figure}

\paragraph{Premature Collapse.} Appendix~\ref{appendix:impact_number_agents} examines the first round alone, where ToS's over-taking is negligible, while a full run requires that same decision at every round. In the harder scenarios half of its runs have ended by round 12 in Pasture and by round 28 in Pollution. In its ordinary rounds the agents harvest almost exactly the sustainable half and leave the resource at its level, and the collapse is a single round in which they harvest nearly all of it at once. The \texttt{action} field on that round still divides the resource by the number of agents, as it does in every round, and divides the whole pool where every other round divides the sustainable half of it. Reasoning that is correct on average is therefore insufficient on an irreversible resource.

\subsection{Failure of Theory of Mind}
\label{appendix:failure_tom}

\paragraph{Prediction Accuracy.}
Table~\ref{tab:tom_prediction} scores each ToM baseline's prediction against the target the other agent then selected. The first step carries no history, and every later step carries the other agent's past actions, the input belief correction is built to read. Competitive accuracy does not improve once that history arrives and sits near $30\%$ thereafter, so reading the other agent's past does not bring the prediction closer to its next choice. Cooperative accuracy stays above $70\%$ at both points. The prediction is therefore accurate only where the agents must select the same target, which is where the prediction-free ReAct already succeeds (Table~\ref{tab:main}), and it is wrong where they must select different ones. ProAgent adds an intention inference and HypMinds a hypothesis test, and on the later Competitive steps both land within a few points of Classic ToM, so the failure lies in predicting a homogeneous agent, whatever the implementation of the prediction.

\begin{table}[H]
\centering
\caption{\textbf{ToM prediction accuracy depends on the setting and does not improve with history.} The percentage of DivvyBench steps on which a baseline predicts the target the other agent then selected. \emph{No History} is the first step, and \emph{With History} every step after it.}
\label{tab:tom_prediction}
\resizebox{\columnwidth}{!}{\input{00_results/table_tom_prediction.tex}}
\end{table}

\paragraph{Prediction Load.}
Figure~\ref{fig:failure_tom} counts how many other agents each ToM baseline names in a single prediction step, across two tests, DivvyBench from 2 to 5 agents and GovSim from 5 to 100. The three baselines all predict the other agents and differ in how many of them they name. On a small team all three name every other agent, sitting on the cap at every DivvyBench size, so the prediction load scales one-for-one with the team. At a hundred agents, HypMinds keeps enumerating and names 92 of the 99, while Classic ToM and ProAgent stop enumerating and name 26 and 6, best-responding to a near-blind guess about the rest. Enumeration can lose the decision itself, since a prediction that long exceeds the generation budget or cannot be parsed, and at a hundred agents this removes 48\% of HypMinds's decisions and 37\% of Classic ToM's, leaving the agent to fall back on a default move. ProAgent holds that loss to 7\% by naming almost no one. ToS conditions on the shared scene at no per-agent cost, so its reasoning does not lengthen as the team grows and still accounts for every agent through the scene.

\begin{figure}[H]
\centering
\includegraphics[width=\linewidth]{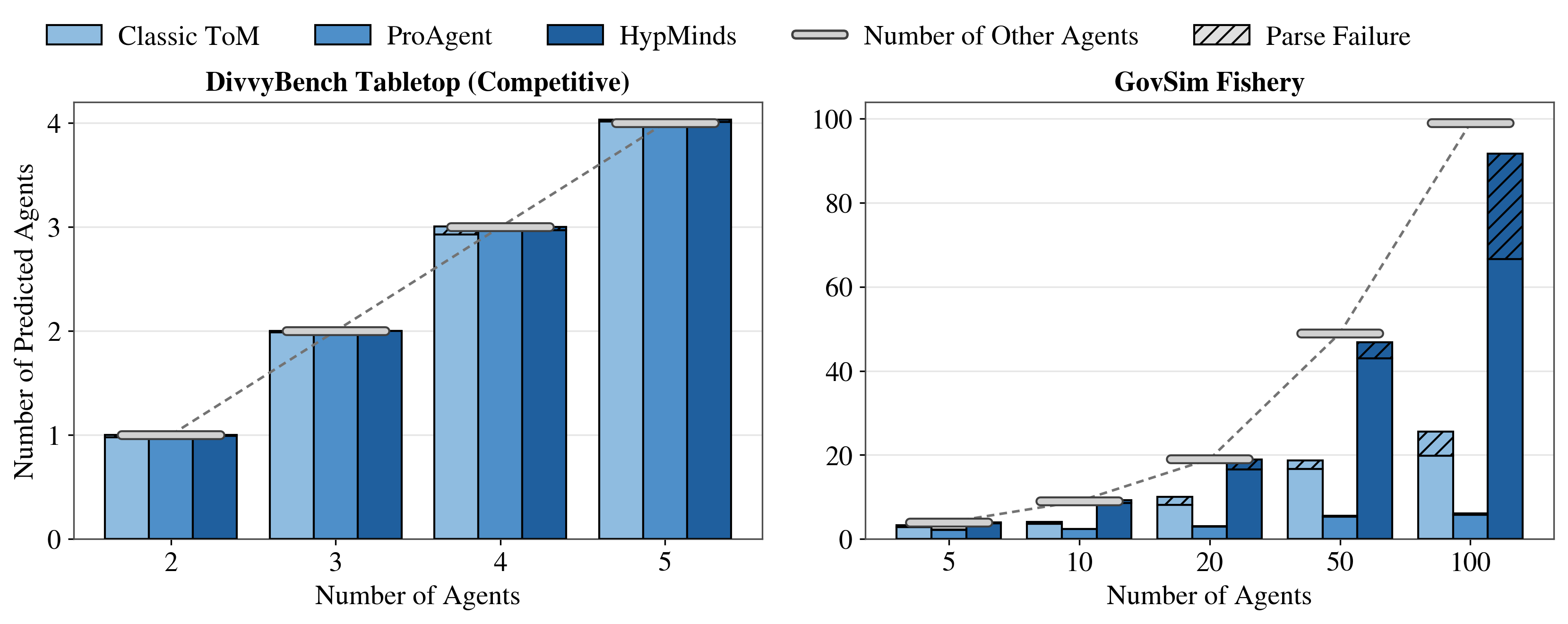}
\caption{\textbf{Predicting a team does not scale.} Each bar is the mean number of other agents a ToM method names in a single prediction step, and the rounded cap marks the number of other agents ($N-1$). The hatched part is what the method predicted in responses whose reasoning failed to parse, so the agent fell back to a default move. Responses cut off at the generation budget are excluded from the mean.}
\label{fig:failure_tom}
\end{figure}

\subsection{Analysis of Per-Agent Distributions}
\label{appendix:delta}

Table~\ref{tab:delta} estimates $\delta_{ij}$, the quantity the two only-if conditions of Proposition~\ref{prop:role} constrain and the proposition leaves unmeasured. The estimate is taken at the first step, which every episode of a scenario enters on the same scene, so the actions recorded over those episodes give each agent's distribution over one set of targets. A team of 2 agents holds a single pair, so $\delta_{ij}$ is at once $\delta_{\min}$ and $\delta_{\max}$ and both conditions constrain one number. The complement $1-\delta_{ij}$ is the overlap of the two distributions and bounds from above the probability that the agents select the same target. A distance of $0$ is necessary for converging and not sufficient, since agents sampling one distribution select the same target only as often as that distribution concentrates on it. In DivvyBench Cooperative every method is near $0$, so the necessary condition is met by homogeneous agents that read nothing. No method approaches $1$ in DivvyBench Competitive, where ToS reaches $0.780$ and CFD, the highest baseline, reaches $0.480$. Across the two settings each method runs one unchanged schema, so any change in $\delta_{ij}$ from one to the other comes from how the method reads the scene. ToS changes it by $0.780$ against $0.440$ for CFD, the largest change among the baselines, so the two readings of Equation~\eqref{eq:tos} account for the change.

\begin{table}[H]
\centering
\caption{\textbf{Only ToS moves the distance toward the value each step requires.} The total variation distance $\delta_{ij}$ between the target distributions of A0 and A1 on the first step of DivvyBench. The arrows give the value Proposition~\ref{prop:role} requires, $1$ where the step's targets have to be divided and $0$ where one of them has to be selected together.}
\label{tab:delta}
\input{00_results/table_delta.tex}
\end{table}

\subsection{Confidence Intervals and Significance Tests}
\label{appendix:significance}

Table~\ref{tab:significance} reports how far ToS leads each baseline, as a two-sided 95\% interval on every gap together with a test of that gap. The interval is a percentile bootstrap over 10{,}000 resamples and the test is a permutation test on the difference of means, with every resample and every permutation kept inside one DivvyBench scenario, one GovSim scenario, or one Overcooked recipe level, so the difficulty spread across those strata cannot enter the null distribution, and Overcooked uses the level-normalized statistic of the main table. ToS's margin is significant in every comparison but ReAct in DivvyBench Cooperative, where ReAct reaches 96.7\% against ToS's 100.0\% and the test returns $p=0.054$ for that 3.3-point gap even though its interval excludes zero, since the permutation test is the more conservative of the two at this sample size.

\begin{table}[H]
\centering
\caption{\textbf{ToS's lead is significant in 35 of the 36 comparisons.} Each baseline row gives ToS's lead over that baseline on the statistic Table~\ref{tab:main} reports, and the brackets hold the two-sided 95\% interval of the number before them. The ToS row gives its own score. \cmark{} marks a gap significant at $p<0.05$ and \xmark{} one that is not.}
\label{tab:significance}
\resizebox{0.7\columnwidth}{!}{\input{00_results/table_significance.tex}}
\end{table}

\newpage
\section{Proofs of Propositions}
\label{appendix:proofs}

\begin{proof}[Proof of Proposition~\ref{prop:blind}]
The team diverges when the $N$ agents select $N$ distinct targets. Each unordered set of $N$ targets is assigned to the $N$ agents in $N!$ orders, so the event has probability $N!\,e_N(p)$. By Maclaurin's inequality, $e_N(p)\le\binom{K}{N}K^{-N}$, with equality at the uniform distribution, so $\max_{p} N!\,e_N(p)=N!\binom{K}{N}K^{-N}=\prod_{n=1}^{N-1}\big(1-\tfrac{n}{K}\big)$. The team converges when all $N$ agents select the same target, an event of probability $\sum_{k}p_k^{N}$, which attains $1$ when $p$ assigns probability $1$ to a single target.
\end{proof}

\begin{proof}[Proof of Proposition~\ref{prop:role}]
Agents $i$ and $j$ select the same target with probability $\sum_k p^{(i)}_k p^{(j)}_k$, since both select target $k$ with probability $p^{(i)}_k p^{(j)}_k$. If the team converges, every pair of agents selects the same target, so $P_{\mathrm{conv}}\le\sum_k p^{(i)}_k p^{(j)}_k$ for every pair $i\ne j$, and applying $p^{(i)}_k p^{(j)}_k\le\min\big(p^{(i)}_k,p^{(j)}_k\big)$ term by term gives $P_{\mathrm{conv}}\le\sum_k\min\big(p^{(i)}_k,p^{(j)}_k\big)=1-\delta_{ij}$. Taking the pair with the largest distance yields $P_{\mathrm{conv}}\le 1-\delta_{\max}$, and $P_{\mathrm{conv}}=1$ requires $\delta_{\max}=0$. If the team diverges, no pair of agents selects the same target, so $P_{\mathrm{div}}\le 1-\sum_k p^{(i)}_k p^{(j)}_k$ for every pair, and by Cauchy-Schwarz together with $\sqrt{p^{(i)}_k p^{(j)}_k}\ge\min\big(p^{(i)}_k,p^{(j)}_k\big)$,
\begin{equation*}
\sum_k p^{(i)}_k p^{(j)}_k\;\ge\;\frac{1}{K}\Big(\sum_k\sqrt{p^{(i)}_k p^{(j)}_k}\Big)^{2}\;\ge\;\frac{1}{K}\Big(\sum_k\min\big(p^{(i)}_k,p^{(j)}_k\big)\Big)^{2}\;=\;\frac{(1-\delta_{ij})^{2}}{K}.
\end{equation*}
Taking the pair with the smallest distance yields $P_{\mathrm{div}}\le 1-(1-\delta_{\min})^{2}/K$, and $P_{\mathrm{div}}=1$ requires $\delta_{\min}=1$. At $\delta_{\min}=0$ this bound reduces to $1-1/K$, which is Proposition~\ref{prop:blind} at $N=2$, so one shared distribution is the case $\delta_{ij}=0$ for every pair.
\end{proof}

\newpage
\section{Experimental Details}
\label{appendix:setup}

\subsection{Configuration and Benchmarks}
\label{appendix:config}

\paragraph{Implementation Details.} Table~\ref{tab:impl} lists the implementation configuration. The settings below are those of the main experiments, and the studies in Appendix~\ref{appendix:further_analysis} vary one of them at a time. We decode with the backbone's thinking mode disabled, so the generation budget goes to the reasoning schema itself. A response the model does not close as valid JSON is repaired when the error is a common one, an unquoted key, a single-quoted string, or a trailing comma, and when it still cannot be read the agent waits.

\begin{table}[H]
\centering
\caption{Implementation details.}
\label{tab:impl}
\resizebox{0.5\columnwidth}{!}{
\begin{tabular}{l l}
\toprule
\textbf{Parameter} & \textbf{Value} \\
\midrule
Backbone & Qwen3.5-35B-A3B \\
Temperature & 0.7 \\
Generation budget & 8192 tokens \\
Max model length & 16384 tokens \\
Thinking mode & Disabled \\
Serving & vLLM on 4 NVIDIA H100 \\
\bottomrule
\end{tabular}
}
\end{table}

\paragraph{Benchmarks.} Table~\ref{tab:bench} lists the three benchmarks, DivvyBench, which we introduce, and the existing GovSim and Overcooked. GovSim gives a group of agents one regenerating pool and asks each of them privately how much to harvest this round, so the pool survives only while the group holds its total under the regrowth. Overcooked puts two cooks in a kitchen whose stations are shared and single-use, so a dish is delivered only when they divide the pipeline between them.

\begin{table}[H]
\centering
\caption{Benchmark details.}
\label{tab:bench}
\resizebox{\columnwidth}{!}{
\begin{tabular}{l l l l}
\toprule
 & \textbf{DivvyBench} & \textbf{GovSim}~\citep{piatti2024cooperate} & \textbf{Overcooked}~\citep{sun2025collab} \\
\midrule
Agents & 2 & 5 & 2 \\
Public role & Leader, Follower & A0 to A4 & Chef, Assistant \\
Scenarios & Tabletop, Airspace, Household & Fishery, Pasture, Pollution & Levels 1--6 \\
Settings & Competitive, Cooperative, Mixed & - & - \\
Action & One target & An integer harvest & One operation (e.g., \texttt{pickup}, \texttt{cut}) \\
Episode budget & 10 steps & 12 rounds & 50 steps \\
Trials & 150 per setting & 50 per scenario & 15 per level \\
Metrics & Success, Stuck & Survival, Total Gain & Throughput \\
\bottomrule
\end{tabular}
}
\end{table}

\subsection{DivvyBench}
\label{appendix:divvybench}

\paragraph{Rules.} A team of $N$ homogeneous agents faces 8 targets. Every agent sees all of them in the same order, may select any target still standing or wait, and all of them act in the same step without seeing what the others chose and without exchanging a message. A target that no agent takes stays where it is, and the step is wasted. A target is either Competitive or Cooperative, which the prompts label light and heavy, and the three settings differ only in the mix.
\begin{itemize}
\item \textbf{Competitive} contains only Competitive targets, each taken when exactly one agent selects it, while two or more agents selecting it collide and take nothing.
\item \textbf{Cooperative} contains only Cooperative targets, each taken only when every agent selects it in the same step, while anything short of that takes nothing.
\item \textbf{Mixed} makes the first 4 targets Cooperative and the rest Competitive, so both rules apply at once.
\end{itemize}

\paragraph{Scenarios.} One set of rules runs in three scenarios, which differ in what the targets are. \textbf{Tabletop} puts 8 balls on a table, named by color, red, blue, green, yellow, orange, purple, pink, and brown. \textbf{Airspace} puts 8 areas in a survey grid, A1, A2, B1, B2, C1, C2, D1, and D2. \textbf{Household} puts 8 indoor locations taken from \citet{zhang2024building}, the bedroom, dishwasher, fridge, kitchen cabinet, living room, microwave, office, and stove. A scenario holds its targets fixed across episodes, so what varies from episode to episode is the sampling, never the layout, and Figure~\ref{fig:trace_divvybench} shows one Tabletop episode of each pure setting.

\paragraph{Metrics.} Two numbers summarize an episode, one for whether the team finished and one for how much of its effort was wasted getting there. In DivvyBench, the theoretical minimum is known in closed form, and an episode scores all or nothing, so each method is measured against the best any team could do as well as against the other methods. Figure~\ref{fig:divvybench_motivation} brackets a step count with two references. A run at the theoretical minimum spends one step on each Cooperative target and shares the Competitive ones out among the agents, which is 4, 8, and 6 steps for the three settings with two agents, while a single agent takes one Competitive target per step and has nobody to collide with, so it completes the Competitive setting in 8 steps however it reasons, and a team that needs more steps performs worse than one agent alone.
\begin{itemize}
\item \textbf{Success} is the fraction of the 150 episodes in which the team takes every target within the 10-step budget, 50 episodes in each scenario. The budget leaves little slack over the theoretical minimum, so finishing is close to all-or-nothing.
\item \textbf{Stuck} is the fraction of steps that take nothing, whether from a collision on a Competitive target, a partial reach on a Cooperative one, or a wait. It keeps the methods apart once success saturates and says how a team failed, since missing the budget with few stuck steps means it was slow while many means it was deadlocked.
\end{itemize}

\begin{figure}[H]
\centering
\begin{minipage}[t]{0.485\textwidth}
\input{10_prompts/trace_divvybench_competitive.tex}
\end{minipage}\hfill
\begin{minipage}[t]{0.485\textwidth}
\input{10_prompts/trace_divvybench_cooperative.tex}
\end{minipage}
\caption{\textbf{One Tabletop episode in each setting.} The agents act at the same step, and the same joint action has opposite outcomes in the two settings. On the left, two agents on one ball collect neither (\xmark) and leave the table unchanged, and on the right they collect it together (\cmark), where a split would collect nothing.}
\label{fig:trace_divvybench}
\end{figure}

\subsection{GovSim}
\label{appendix:govsim}

\paragraph{Rules.} $N$ agents share one regenerating pool for at most 12 rounds. The pool holds up to 100 units, and at the start of a round every agent privately names an integer between 0 and 100, and that many units are removed at once. What remains then doubles, capped at the capacity, and the pool collapses once it falls below 5 units. Each agent is told to maximize what it harvests over the long run, and its observation carries the current pool level and the harvests every agent made in the earlier rounds. The three scenarios put the same arithmetic in three framings, a lake of fish, a public pasture, and a river shared by factories. We remove two things GovSim gives the agents. \emph{(i)} The communication period, in which every agent's harvest is announced and the agents negotiate and persuade one another before the next round. We drop the sentences that announce it and never run the exchange, so nothing an agent writes reaches another. \emph{(ii)} The universalization hint, a line telling an agent what the resource would do if every agent took what it is about to take. We never inject it, since it would hand the agent, from outside, the reasoning a coordination method is meant to produce.

\paragraph{Metrics.} Two numbers summarize a run, one for how long the resource lasts and one for how much comes out of it.
\begin{itemize}
\item \textbf{Survival} is the number of rounds the pool stays above the collapse threshold, reported as a percentage of the 12-round horizon.
\item \textbf{Total Gain} is the units harvested over the whole run, summed across the agents and the rounds.
\end{itemize}

\subsection{Overcooked}
\label{appendix:overcooked}

\paragraph{Rules.} Two agents cook in one kitchen for at most 50 steps. A recipe is a pipeline running from fetching an ingredient through cutting or blending it, cooking or baking it, plating it, and delivering the dish, and every station is single-use, so a station one agent is working is closed to the other for that step. Each of the six difficulty levels holds five recipes, and each recipe is run three times, giving 15 episodes per level. We change two things. \emph{(i)} We run the open kitchen, in which every station is within reach of either agent and each of them holds the full recipe, so no agent is confined to a part of the pipeline by what it can reach. The layout is close to the Cramped Room of \citet{carroll2019utility}, where the agents share one room and contend for the same stations. \emph{(ii)} Collab-Overcooked's action space carries a request action beside the cooking operations, which hands another agent an operation to perform and sends it a message. We remove that action, so nothing an agent can do addresses another and the shared kitchen state is all they hold in common.

\paragraph{Metrics.} Every method delivers the dish on almost every episode, so success does not separate them and we report throughput, the number of dishes delivered within the 50 steps. The main table aggregates it over the six levels in two ways.
\begin{itemize}
\item \textbf{Raw} is the unweighted mean of the six level means. Yield differs widely across levels, 6.47 dishes at level 1 against 1.20 at level 6 for ToS, so the easier levels dominate it.
\item \textbf{$\times$ReAct} divides each level mean by ReAct's mean at the same level before taking that mean, placing the levels on a common scale. The two levels above then read 1.31 and 1.12 for ToS.
\end{itemize}

\newpage
\section{Prompt Templates}
\label{appendix:prompt}

\subsection{DivvyBench}
\label{appendix:prompt:divvybench}

The three DivvyBench settings use the same task prompt apart from the rules block. The figures show the task description, the rules, and the public role, and at each step the prompt continues with the current state, the legal actions, and the history, followed by the method's reasoning schema.

\begin{figure}[H]
\input{10_prompts/prompt_divvybench_competitive.tex}
\vspace{-10pt}
\caption{Task prompt for DivvyBench Competitive in Tabletop.}
\label{fig:prompt_divvybench_competitive}
\end{figure}
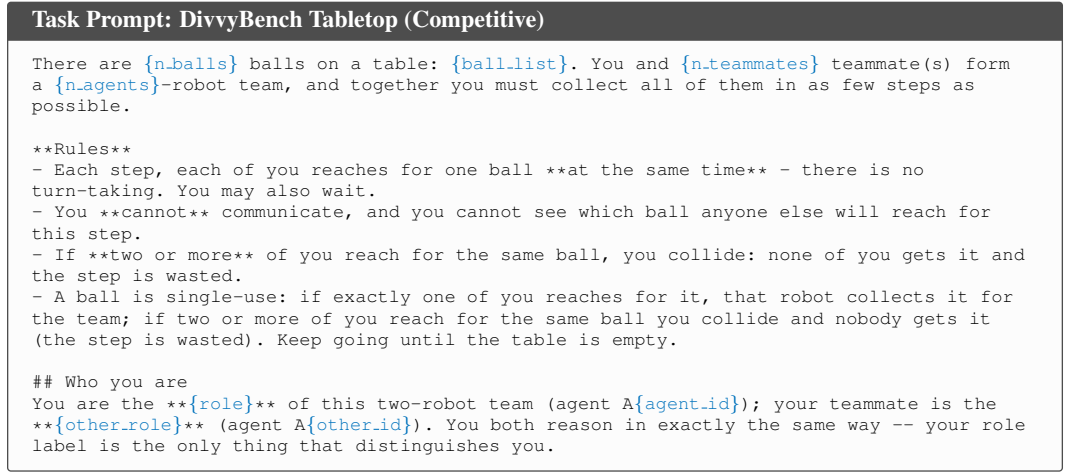

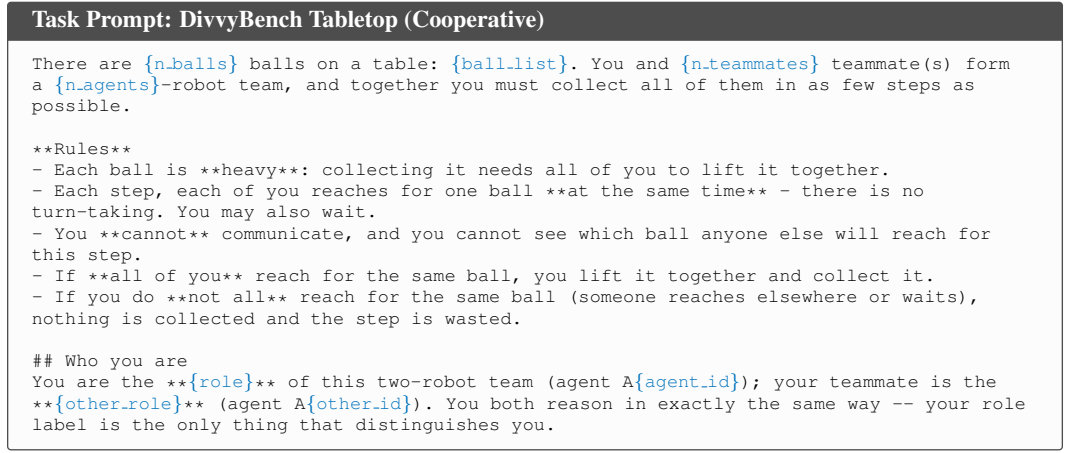
\begin{figure}[H]
\input{10_prompts/prompt_divvybench_cooperative.tex}
\vspace{-10pt}
\caption{Task prompt for DivvyBench Cooperative in Tabletop.}
\label{fig:prompt_divvybench_cooperative}
\end{figure}

\begin{figure}[H]
\input{10_prompts/prompt_divvybench_mixed.tex}
\vspace{-10pt}
\caption{Task prompt for DivvyBench Mixed in Tabletop.}
\label{fig:prompt_divvybench_mixed}
\end{figure}

\begin{figure}[H]
\input{10_prompts/prompt_divvybench_airspace.tex}
\vspace{-10pt}
\caption{Task prompt for DivvyBench Competitive in Airspace.}
\label{fig:prompt_divvybench_airspace}
\end{figure}

\begin{figure}[H]
\input{10_prompts/prompt_divvybench_household.tex}
\vspace{-10pt}
\caption{Task prompt for DivvyBench Competitive in Household.}
\label{fig:prompt_divvybench_household}
\end{figure}

\clearpage
\subsection{Theory of Scene}
\label{appendix:prompt:tos}

Figure~\ref{fig:prompt_tos} gives the ToS reasoning schema, the block appended to the task prompt on every benchmark. Section~\ref{sec:method:tos} writes that schema with one \texttt{action} field. The implementation splits that step in two, a \texttt{plan} field that states the division the readings name and an \texttt{action} field constrained to copy one action verbatim from the currently legal ones, because a schema that emits the action directly returns actions outside that set. Every schema in the comparison carries the same split.

\begin{figure}[H]
\input{10_prompts/prompt_tos.tex}
\vspace{-10pt}
\caption{Prompt for ToS.}
\label{fig:prompt_tos}
\end{figure}

\newpage
\section{Theory of Scene Reasoning Examples}
\label{appendix:reasoning_example}

Figures~\ref{fig:trace_tos_divvybench} to~\ref{fig:trace_tos_overcooked} reproduce ToS reasoning verbatim from the runs, one per benchmark, with the role and task readings highlighted. Role gating yields both of its readings, since on DivvyBench and GovSim the agents read ownership as overlapping and have to settle a division, while on Overcooked they read it as already divided and each advances a stage of the pipeline. The three couplings all appear, \emph{Joint} and \emph{Exclusive} in DivvyBench Mixed, \emph{Exclusive} on GovSim, and \emph{Sequential} on Overcooked.

\begin{figure}[H]
\input{10_prompts/trace_tos_divvybench.tex}
\vspace{-10pt}
\caption{ToS reasoning on DivvyBench Tabletop (Mixed), where one table carries both couplings. On Step 1 the agents read the Cooperative balls as \emph{Joint} and commit to the first one in the shared order, since a Cooperative ball is taken only when all agents select it together. On Step 5 the Cooperative balls are gone, the agents read the remaining Competitive balls as \emph{Exclusive}, and they split them without exchanging a message.}
\label{fig:trace_tos_divvybench}
\end{figure}

\begin{figure}[H]
\input{10_prompts/trace_tos_govsim.tex}
\vspace{-10pt}
\caption{ToS reasoning on GovSim (Fishery). The agents read the pool as shared and each asks for its per-agent share of what the pool can regrow, so the total stays sustainable without anyone proposing it. Only A0 is expanded, and the other four reach the same two readings and the same catch.}
\label{fig:trace_tos_govsim}
\end{figure}

\begin{figure}[H]
\input{10_prompts/trace_tos_overcooked.tex}
\vspace{-10pt}
\caption{ToS reasoning on Overcooked. The recipe is a chain and the public role gives each agent its own stations, so ownership reads as divided and no agent settles a division. On Step 1, A1 opens the chain by fetching the pepper while A0, whose stations have nothing to work on yet, fetches the dish the order will be served on. On Step 2, A1 loads the board and A0 waits, since nothing in its own part is ready.}
\label{fig:trace_tos_overcooked}
\end{figure}

\end{document}

%% file: 00_results/table_main.tex
\begin{tabular}{c lcccccccc}
\toprule
\multirow{9}{*}{\rotatebox[origin=c]{90}{\textbf{DivvyBench}}} &  & \multicolumn{2}{c}{Competitive} & \multicolumn{2}{c}{Cooperative} & \multicolumn{2}{c}{Mixed} & \multicolumn{2}{c}{Avg.} \\
\cmidrule(lr){3-4}\cmidrule(lr){5-6}\cmidrule(lr){7-8}\cmidrule(lr){9-10}
 & Method & Success (\%)  $\uparrow$ & Stuck (\%)  $\downarrow$ & Success (\%)  $\uparrow$ & Stuck (\%)  $\downarrow$ & Success (\%)  $\uparrow$ & Stuck (\%)  $\downarrow$ & Success (\%)  $\uparrow$ & Stuck (\%)  $\downarrow$ \\
\cmidrule{2-10}
 & ReAct & 40.7 \std{4.0} & 48.7 \std{1.4} & 96.7 \std{1.5} & 2.5 \std{0.7} & 56.7 \std{4.1} & 26.9 \std{1.1} & 64.7 \std{2.3} & 26.0 \std{1.1} \\
 & Focal & 47.3 \std{4.1} & 45.4 \std{1.4} & 42.0 \std{4.0} & 36.6 \std{2.4} & 50.0 \std{4.1} & 29.0 \std{1.1} & 46.4 \std{2.4} & 37.0 \std{1.1} \\
 & CFD & 57.3 \std{4.1} & 31.8 \std{1.8} & 66.7 \std{3.9} & 20.5 \std{2.1} & 36.0 \std{3.9} & 33.9 \std{1.7} & 53.3 \std{2.4} & 28.7 \std{1.1} \\
 & Classic ToM & 72.7 \std{3.7} & 35.6 \std{1.5} & 86.7 \std{2.8} & 8.2 \std{1.6} & 54.0 \std{4.1} & 27.1 \std{1.6} & 71.1 \std{2.1} & 23.6 \std{1.1} \\
 & ProAgent & 68.7 \std{3.8} & 37.2 \std{1.7} & 72.0 \std{3.7} & 20.1 \std{2.7} & 34.7 \std{3.9} & 35.0 \std{1.5} & 58.4 \std{2.3} & 30.8 \std{1.2} \\
 & HypMinds & 66.0 \std{3.9} & 39.5 \std{1.7} & 62.0 \std{4.0} & 26.6 \std{2.9} & 47.3 \std{4.1} & 30.9 \std{1.5} & 58.4 \std{2.3} & 32.3 \std{1.2} \\
 & \cellcolor{ours}ToS (Ours) & \cellcolor{ours}\textbf{99.3} \std{0.7} & \cellcolor{ours}\textbf{15.4} \std{1.3} & \cellcolor{ours}\textbf{100.0} \std{0.0} & \cellcolor{ours}\textbf{0.0} \std{0.0} & \cellcolor{ours}\textbf{99.3} \std{0.7} & \cellcolor{ours}\textbf{9.1} \std{0.9} & \cellcolor{ours}\textbf{99.6} \std{0.3} & \cellcolor{ours}\textbf{8.2} \std{0.6} \\
\midrule
\multirow{9}{*}{\rotatebox[origin=c]{90}{\textbf{GovSim}}} &  & \multicolumn{2}{c}{Fishery} & \multicolumn{2}{c}{Pasture} & \multicolumn{2}{c}{Pollution} & \multicolumn{2}{c}{Avg.} \\
\cmidrule(lr){3-4}\cmidrule(lr){5-6}\cmidrule(lr){7-8}\cmidrule(lr){9-10}
 & Method & Survival (\%)  $\uparrow$ & Total Gain  $\uparrow$ & Survival (\%)  $\uparrow$ & Total Gain  $\uparrow$ & Survival (\%)  $\uparrow$ & Total Gain  $\uparrow$ & Survival (\%)  $\uparrow$ & Total Gain  $\uparrow$ \\
\cmidrule{2-10}
 & ReAct & 46.5 \std{4.6} & 226 \std{20} & 21.3 \std{1.9} & 129 \std{4} & 11.7 \std{1.0} & 109 \std{3} & 26.5 \std{2.1} & 155 \std{8} \\
 & Focal & 36.8 \std{4.8} & 207 \std{21} & 14.8 \std{1.5} & 120 \std{5} & 8.8 \std{0.4} & 102 \std{2} & 20.2 \std{1.9} & 143 \std{8} \\
 & CFD & 84.2 \std{3.3} & 525 \std{23} & 49.7 \std{3.8} & \textbf{257} \std{19} & 43.7 \std{4.7} & 244 \std{22} & 59.2 \std{2.7} & 342 \std{16} \\
 & Classic ToM & 63.3 \std{4.1} & 357 \std{23} & 19.5 \std{1.4} & 144 \std{5} & 13.7 \std{1.1} & 122 \std{5} & 32.2 \std{2.3} & 207 \std{12} \\
 & ProAgent & 29.2 \std{3.8} & 143 \std{8} & 11.8 \std{0.8} & 107 \std{2} & 8.3 \std{0.0} & 100 \std{0} & 16.4 \std{1.5} & 117 \std{3} \\
 & HypMinds & 80.7 \std{3.8} & 448 \std{23} & 35.8 \std{2.2} & 189 \std{6} & 23.0 \std{2.9} & 147 \std{9} & 46.5 \std{2.7} & 262 \std{14} \\
 & \cellcolor{ours}ToS (Ours) & \cellcolor{ours}\textbf{99.0} \std{0.4} & \cellcolor{ours}\textbf{550} \std{10} & \cellcolor{ours}\textbf{72.0} \std{4.8} & \cellcolor{ours}254 \std{17} & \cellcolor{ours}\textbf{84.5} \std{4.3} & \cellcolor{ours}\textbf{397} \std{21} & \cellcolor{ours}\textbf{85.2} \std{2.3} & \cellcolor{ours}\textbf{400} \std{14} \\
\midrule
\multirow{9}{*}{\rotatebox[origin=c]{90}{\textbf{Overcooked}}} &  & \multicolumn{6}{c}{Throughput (dishes)  $\uparrow$} & \multicolumn{2}{c}{Avg.} \\
\cmidrule(lr){3-8}\cmidrule(lr){9-10}
 & Method & Level 1 & Level 2 & Level 3 & Level 4 & Level 5 & Level 6 & Raw  $\uparrow$ & $\times$ReAct  $\uparrow$ \\
\cmidrule{2-10}
 & ReAct & 4.93 \std{0.18} & 3.47 \std{0.22} & 2.13 \std{0.19} & 2.27 \std{0.21} & 0.60 \std{0.13} & 1.07 \std{0.12} & 2.41 \std{0.17} & 1.00 \std{0.05} \\
 & Focal & 5.40 \std{0.25} & 3.07 \std{0.25} & 2.07 \std{0.18} & 2.00 \std{0.17} & 1.20 \std{0.14} & 0.87 \std{0.13} & 2.43 \std{0.18} & 1.11 \std{0.07} \\
 & CFD & 4.80 \std{0.22} & 2.47 \std{0.17} & 1.80 \std{0.11} & 2.33 \std{0.16} & 1.47 \std{0.13} & 1.00 \std{0.17} & 2.31 \std{0.14} & 1.16 \std{0.08} \\
 & Classic ToM & 5.53 \std{0.29} & 4.47 \std{0.19} & 2.73 \std{0.15} & 3.27 \std{0.23} & 1.33 \std{0.19} & 1.20 \std{0.11} & 3.09 \std{0.18} & 1.41 \std{0.07} \\
 & ProAgent & 6.00 \std{0.58} & 3.73 \std{0.33} & 2.87 \std{0.19} & \textbf{3.80} \std{0.14} & 1.27 \std{0.12} & 0.80 \std{0.17} & 3.08 \std{0.22} & 1.36 \std{0.07} \\
 & HypMinds & 4.67 \std{0.58} & 5.00 \std{0.35} & 3.07 \std{0.18} & 3.27 \std{0.21} & 1.47 \std{0.17} & \textbf{1.33} \std{0.13} & 3.13 \std{0.19} & 1.49 \std{0.08} \\
 & \cellcolor{ours}ToS (Ours) & \cellcolor{ours}\textbf{6.47} \std{0.39} & \cellcolor{ours}\textbf{5.73} \std{0.44} & \cellcolor{ours}\textbf{3.40} \std{0.24} & \cellcolor{ours}3.53 \std{0.27} & \cellcolor{ours}\textbf{1.67} \std{0.13} & \cellcolor{ours}1.20 \std{0.17} & \cellcolor{ours}\textbf{3.67} \std{0.24} & \cellcolor{ours}\textbf{1.67} \std{0.08} \\
\bottomrule
\end{tabular}

%% file: 00_results/table_ablation.tex
\begin{tabular}{l ccc cc c}
\toprule
 & \multicolumn{3}{c}{\textbf{DivvyBench} - Success (\%) $\uparrow$} & \multicolumn{2}{c}{\textbf{GovSim}} & \multicolumn{1}{c}{\textbf{Overcooked}} \\
\cmidrule(lr){2-4}\cmidrule(lr){5-6}\cmidrule(lr){7-7}
Method & Competitive & Cooperative & Mixed & Survival (\%) $\uparrow$ & Total Gain $\uparrow$ & $\times$ReAct $\uparrow$ \\
\midrule
ToS (Ours) & \textbf{99.3} \std{0.7} & \textbf{100.0} \std{0.0} & 99.3 \std{0.7} & \textbf{85.2} \std{2.3} & 400 \std{14} & \textbf{1.67} \std{0.08} \\
\multicolumn{7}{c}{\cellcolor[gray]{0.85}\rule[-1ex]{0pt}{3ex}\textit{Ablation Study}} \\
\quad w/o \texttt{role} & 97.3 \std{1.3} & \textbf{100.0} \std{0.0} & 99.3 \std{0.7} & \cellcolor{bad}58.1 \std{3.0} & \cellcolor{bad}242 \std{10} & \cellcolor{bad}1.45 \std{0.09} \\
\quad w/o \texttt{task} & \cellcolor{bad}83.3 \std{3.1} & \cellcolor{bad}76.0 \std{3.5} & \cellcolor{bad}18.0 \std{3.1} & \cellcolor{bad}67.9 \std{3.0} & \cellcolor{bad}344 \std{17} & 1.61 \std{0.07} \\
\multicolumn{7}{c}{\cellcolor[gray]{0.85}\rule[-1ex]{0pt}{3ex}\textit{Robustness Study}} \\
\quad \texttt{task}$\rightarrow$\texttt{role} & \textbf{99.3} \std{0.7} & \textbf{100.0} \std{0.0} & \textbf{100.0} \std{0.0} & 80.3 \std{2.7} & 375 \std{14} & 1.51 \std{0.07} \\
\quad w/o $\mathbf{a}_{-i}^{<t}$ & 98.7 \std{0.9} & \textbf{100.0} \std{0.0} & \textbf{100.0} \std{0.0} & 81.5 \std{2.4} & \textbf{407} \std{14} & 1.54 \std{0.08} \\
\bottomrule
\end{tabular}

%% file: 00_results/table_tom_tos.tex
\begin{tabular}{l ccc}
\toprule
 & \multicolumn{3}{c}{\textbf{DivvyBench} - Success (\%) $\uparrow$} \\
\cmidrule(lr){2-4}
Method & Competitive & Cooperative & Mixed \\
\midrule
Classic ToM & 72.7 \std{3.7} & 86.7 \std{2.8} & 54.0 \std{4.1} \\
\quad + \texttt{role} & 64.7 \std{3.9} & 80.0 \std{3.3} & 47.3 \std{4.1} \\
\quad + \texttt{task} & 91.3 \std{2.3} & 68.7 \std{3.8} & 82.7 \std{3.1} \\
\quad + \texttt{role}+\texttt{task} & 88.7 \std{2.6} & 68.7 \std{3.8} & 78.7 \std{3.4} \\
\cellcolor{ours}ToS (Ours) & \cellcolor{ours}\textbf{99.3} \std{0.7} & \cellcolor{ours}\textbf{100.0} \std{0.0} & \cellcolor{ours}\textbf{99.3} \std{0.7} \\
\midrule
 & \multicolumn{2}{c}{\textbf{GovSim}} & \multicolumn{1}{c}{\textbf{Overcooked}} \\
\cmidrule(lr){2-3}\cmidrule(lr){4-4}
 & Survival (\%) $\uparrow$ & Total Gain $\uparrow$ & $\times$ReAct $\uparrow$ \\
\midrule
Classic ToM & 32.2 \std{2.3} & 207 \std{12} & 1.41 \std{0.07} \\
\quad + \texttt{role} & 31.2 \std{2.3} & 191 \std{10} & 1.50 \std{0.09} \\
\quad + \texttt{task} & 31.7 \std{2.5} & 211 \std{13} & 1.63 \std{0.08} \\
\quad + \texttt{role}+\texttt{task} & 39.1 \std{2.8} & 247 \std{15} & 1.66 \std{0.08} \\
\cellcolor{ours}ToS (Ours) & \cellcolor{ours}\textbf{85.2} \std{2.3} & \cellcolor{ours}\textbf{400} \std{14} & \cellcolor{ours}\textbf{1.67} \std{0.08} \\
\bottomrule
\end{tabular}

%% file: 00_results/table_1bit.tex
\begin{tabular}{c l cc cc cc}
\toprule
& & \multicolumn{2}{c}{Competitive} & \multicolumn{2}{c}{Cooperative} & \multicolumn{2}{c}{Mixed} \\
\cmidrule(lr){3-4}\cmidrule(lr){5-6}\cmidrule(lr){7-8}
Temp & Public Role & Success (\%) $\uparrow$ & Stuck (\%) $\downarrow$ & Success (\%) $\uparrow$ & Stuck (\%) $\downarrow$ & Success (\%) $\uparrow$ & Stuck (\%) $\downarrow$ \\
\midrule
\multirow{3}{*}{\textbf{0.0}} & None & \cellcolor{bad}0.7 \std{0.7} & \cellcolor{bad}80.5 \std{1.8} & 100.0 \std{0.0} & 0.0 \std{0.0} & \cellcolor{bad}8.0 \std{2.2} & \cellcolor{bad}48.9 \std{1.2} \\
 & 1 Bit (index) & \cellcolor{bad}92.7 \std{2.1} & \cellcolor{bad}24.4 \std{1.9} & 100.0 \std{0.0} & 0.0 \std{0.0} & 99.3 \std{0.7} & \cellcolor{bad}9.8 \std{0.9} \\
 & 1 Bit (role) & 100.0 \std{0.0} & 9.3 \std{1.0} & 100.0 \std{0.0} & 0.0 \std{0.0} & 100.0 \std{0.0} & 6.4 \std{0.8} \\
\midrule
\multirow{3}{*}{\textbf{0.7}} & None & \cellcolor{bad}12.7 \std{2.7} & \cellcolor{bad}53.1 \std{1.4} & 100.0 \std{0.0} & 0.2 \std{0.1} & \cellcolor{bad}28.0 \std{3.7} & \cellcolor{bad}33.0 \std{1.1} \\
 & 1 Bit (index) & 99.3 \std{0.7} & \cellcolor{bad}19.5 \std{1.3} & 100.0 \std{0.0} & 0.2 \std{0.1} & 98.7 \std{0.9} & \cellcolor{bad}13.7 \std{1.0} \\
 & 1 Bit (role) & 99.3 \std{0.7} & 15.4 \std{1.3} & 100.0 \std{0.0} & 0.0 \std{0.0} & 99.3 \std{0.7} & 9.1 \std{0.9} \\
\midrule
\multirow{3}{*}{\textbf{1.0}} & None & \cellcolor{bad}11.3 \std{2.6} & \cellcolor{bad}51.5 \std{1.3} & 100.0 \std{0.0} & 0.5 \std{0.2} & \cellcolor{bad}22.7 \std{3.4} & \cellcolor{bad}34.0 \std{1.0} \\
 & 1 Bit (index) & 99.3 \std{0.7} & \cellcolor{bad}20.4 \std{1.4} & 100.0 \std{0.0} & 0.5 \std{0.2} & 98.7 \std{0.9} & \cellcolor{bad}13.7 \std{1.0} \\
 & 1 Bit (role) & 99.3 \std{0.7} & 15.1 \std{1.3} & 100.0 \std{0.0} & 0.2 \std{0.1} & 98.0 \std{1.1} & 11.4 \std{0.9} \\
\bottomrule
\end{tabular}

%% file: 00_results/impact_observation_asymmetry.tex
\begin{tabular}{c l cc cc cc}
\toprule
& & \multicolumn{2}{c}{Competitive} & \multicolumn{2}{c}{Cooperative} & \multicolumn{2}{c}{Mixed} \\
\cmidrule(lr){3-4}\cmidrule(lr){5-6}\cmidrule(lr){7-8}
& Method & Success (\%) $\uparrow$ & Stuck (\%) $\downarrow$ & Success (\%) $\uparrow$ & Stuck (\%) $\downarrow$ & Success (\%) $\uparrow$ & Stuck (\%) $\downarrow$ \\
\midrule
\multirow{7}{*}{\rotatebox[origin=c]{90}{\textbf{Shared}}}
 & ReAct & 40.7 \std{4.0} & 48.7 \std{1.4} & 96.7 \std{1.5} & 2.5 \std{0.7} & 56.7 \std{4.1} & 26.9 \std{1.1} \\
 & Focal & 47.3 \std{4.1} & 45.4 \std{1.4} & 42.0 \std{4.0} & 36.6 \std{2.4} & 50.0 \std{4.1} & 29.0 \std{1.1} \\
 & CFD & 57.3 \std{4.1} & 31.8 \std{1.8} & 66.7 \std{3.9} & 20.5 \std{2.1} & 36.0 \std{3.9} & 33.9 \std{1.7} \\
 & Classic ToM & 72.7 \std{3.7} & 35.6 \std{1.5} & 86.7 \std{2.8} & 8.2 \std{1.6} & 54.0 \std{4.1} & 27.1 \std{1.6} \\
 & ProAgent & 68.7 \std{3.8} & 37.2 \std{1.7} & 72.0 \std{3.7} & 20.1 \std{2.7} & 34.7 \std{3.9} & 35.0 \std{1.5} \\
 & HypMinds & 66.0 \std{3.9} & 39.5 \std{1.7} & 62.0 \std{4.0} & 26.6 \std{2.9} & 47.3 \std{4.1} & 30.9 \std{1.5} \\
 & ToS (Ours) & \textbf{99.3} \std{0.7} & \textbf{15.4} \std{1.3} & \textbf{100.0} \std{0.0} & \textbf{0.0} \std{0.0} & \textbf{99.3} \std{0.7} & \textbf{9.1} \std{0.9} \\
\midrule
\multirow{7}{*}{\rotatebox[origin=c]{90}{\textbf{Reversed}}}
 & ReAct & 94.7 \std{1.8} & \textbf{12.7} \std{1.5} & \cellcolor{bad}0.0 \std{0.0} & \cellcolor{bad}78.4 \std{1.2} & \cellcolor{bad}5.3 \std{1.8} & \cellcolor{bad}55.4 \std{1.0} \\
 & Focal & 67.3 \std{3.8} & 32.9 \std{1.5} & \cellcolor{bad}0.7 \std{0.7} & \cellcolor{bad}75.5 \std{1.2} & \cellcolor{bad}12.0 \std{2.7} & \cellcolor{bad}51.0 \std{1.3} \\
 & CFD & 62.0 \std{4.0} & 31.3 \std{1.4} & \cellcolor{bad}4.7 \std{1.7} & \cellcolor{bad}68.2 \std{1.8} & \cellcolor{bad}3.3 \std{1.5} & \cellcolor{bad}62.8 \std{1.3} \\
 & Classic ToM & 85.3 \std{2.9} & 25.3 \std{1.6} & \cellcolor{bad}1.3 \std{0.9} & \cellcolor{bad}80.0 \std{1.3} & \cellcolor{bad}2.0 \std{1.1} & \cellcolor{bad}58.9 \std{1.0} \\
 & ProAgent & 77.3 \std{3.4} & 28.6 \std{1.8} & \cellcolor{bad}1.3 \std{0.9} & \cellcolor{bad}84.5 \std{1.3} & \cellcolor{bad}7.3 \std{2.1} & \cellcolor{bad}55.1 \std{1.3} \\
 & HypMinds & 89.3 \std{2.5} & 19.5 \std{1.6} & \cellcolor{bad}4.0 \std{1.6} & \cellcolor{bad}79.5 \std{1.8} & \cellcolor{bad}4.0 \std{1.6} & \cellcolor{bad}57.1 \std{1.0} \\
 & ToS (Ours) & \textbf{99.3} \std{0.7} & 16.2 \std{1.3} & \textbf{93.3} \std{2.0} & \textbf{7.1} \std{0.9} & \textbf{83.3} \std{3.1} & \textbf{23.9} \std{1.6} \\
\midrule
\multirow{7}{*}{\rotatebox[origin=c]{90}{\textbf{Shuffled}}}
 & ReAct & 92.0 \std{2.2} & 20.8 \std{1.7} & \cellcolor{bad}0.7 \std{0.7} & \cellcolor{bad}77.1 \std{1.3} & \cellcolor{bad}4.0 \std{1.6} & \cellcolor{bad}61.5 \std{1.2} \\
 & Focal & 61.3 \std{4.0} & 34.1 \std{1.5} & \cellcolor{bad}20.0 \std{3.3} & \cellcolor{bad}61.2 \std{2.4} & \cellcolor{bad}6.7 \std{2.0} & \cellcolor{bad}59.2 \std{1.3} \\
 & CFD & 66.0 \std{3.9} & 28.5 \std{1.4} & \cellcolor{bad}18.0 \std{3.1} & \cellcolor{bad}60.1 \std{2.5} & \cellcolor{bad}2.0 \std{1.1} & \cellcolor{bad}64.8 \std{1.2} \\
 & Classic ToM & 92.0 \std{2.2} & 20.7 \std{1.5} & \cellcolor{bad}0.0 \std{0.0} & \cellcolor{bad}84.6 \std{0.9} & \cellcolor{bad}4.0 \std{1.6} & \cellcolor{bad}56.9 \std{1.1} \\
 & ProAgent & 69.3 \std{3.8} & 32.8 \std{1.9} & \cellcolor{bad}0.0 \std{0.0} & \cellcolor{bad}86.9 \std{0.9} & \cellcolor{bad}6.7 \std{2.0} & \cellcolor{bad}59.2 \std{1.4} \\
 & HypMinds & 86.0 \std{2.8} & 26.1 \std{1.7} & \cellcolor{bad}2.7 \std{1.3} & \cellcolor{bad}82.5 \std{1.4} & \cellcolor{bad}4.7 \std{1.7} & \cellcolor{bad}57.1 \std{1.1} \\
 & ToS (Ours) & \textbf{99.3} \std{0.7} & \textbf{15.8} \std{1.3} & \textbf{92.0} \std{2.2} & \textbf{8.1} \std{0.8} & \textbf{80.7} \std{3.2} & \textbf{26.3} \std{1.8} \\
\bottomrule
\end{tabular}

%% file: 00_results/impact_backbone.tex
\begin{tabular}{c l cc cc cc cc}
\toprule
& & \multicolumn{2}{c}{Competitive} & \multicolumn{2}{c}{Cooperative} & \multicolumn{2}{c}{Mixed} & \multicolumn{2}{c}{Avg.} \\
\cmidrule(lr){3-4}\cmidrule(lr){5-6}\cmidrule(lr){7-8}\cmidrule(lr){9-10}
& Method & Success (\%)  $\uparrow$ & Stuck (\%)  $\downarrow$ & Success (\%)  $\uparrow$ & Stuck (\%)  $\downarrow$ & Success (\%)  $\uparrow$ & Stuck (\%)  $\downarrow$ & Success (\%)  $\uparrow$ & Stuck (\%)  $\downarrow$ \\
\midrule
\multirow{7}{*}{\rotatebox[origin=c]{90}{\textbf{Qwen3.5-35B}}}
 & ReAct & 40.7 \std{4.0} & 48.7 \std{1.4} & 96.7 \std{1.5} & 2.5 \std{0.7} & 56.7 \std{4.1} & 26.9 \std{1.1} & 64.7 \std{2.3} & 26.0 \std{1.1} \\
 & Focal & 47.3 \std{4.1} & 45.4 \std{1.4} & 42.0 \std{4.0} & 36.6 \std{2.4} & 50.0 \std{4.1} & 29.0 \std{1.1} & 46.4 \std{2.4} & 37.0 \std{1.1} \\
 & CFD & 57.3 \std{4.1} & 31.8 \std{1.8} & 66.7 \std{3.9} & 20.5 \std{2.1} & 36.0 \std{3.9} & 33.9 \std{1.7} & 53.3 \std{2.4} & 28.7 \std{1.1} \\
 & Classic ToM & 72.7 \std{3.7} & 35.6 \std{1.5} & 86.7 \std{2.8} & 8.2 \std{1.6} & 54.0 \std{4.1} & 27.1 \std{1.6} & 71.1 \std{2.1} & 23.6 \std{1.1} \\
 & ProAgent & 68.7 \std{3.8} & 37.2 \std{1.7} & 72.0 \std{3.7} & 20.1 \std{2.7} & 34.7 \std{3.9} & 35.0 \std{1.5} & 58.4 \std{2.3} & 30.8 \std{1.2} \\
 & HypMinds & 66.0 \std{3.9} & 39.5 \std{1.7} & 62.0 \std{4.0} & 26.6 \std{2.9} & 47.3 \std{4.1} & 30.9 \std{1.5} & 58.4 \std{2.3} & 32.3 \std{1.2} \\
 & \cellcolor{ours}ToS (Ours) & \cellcolor{ours}\textbf{99.3} \std{0.7} & \cellcolor{ours}\textbf{15.4} \std{1.3} & \cellcolor{ours}\textbf{100.0} \std{0.0} & \cellcolor{ours}\textbf{0.0} \std{0.0} & \cellcolor{ours}\textbf{99.3} \std{0.7} & \cellcolor{ours}\textbf{9.1} \std{0.9} & \cellcolor{ours}\textbf{99.6} \std{0.3} & \cellcolor{ours}\textbf{8.2} \std{0.6} \\
\midrule
\multirow{7}{*}{\rotatebox[origin=c]{90}{\textbf{Gemma-4-26B}}}
 & ReAct & 48.7 \std{4.1} & 51.4 \std{2.2} & \textbf{100.0} \std{0.0} & 0.7 \std{0.2} & 45.3 \std{4.1} & 34.3 \std{1.3} & 64.7 \std{2.3} & 28.8 \std{1.3} \\
 & Focal & 34.0 \std{3.9} & 62.0 \std{1.9} & 86.7 \std{2.8} & 9.5 \std{1.2} & 30.0 \std{3.8} & 40.7 \std{1.0} & 50.2 \std{2.4} & 37.4 \std{1.3} \\
 & CFD & 94.0 \std{1.9} & 27.7 \std{1.8} & 91.3 \std{2.3} & 4.9 \std{1.0} & 61.3 \std{4.0} & 28.4 \std{1.5} & 82.2 \std{1.8} & 20.3 \std{1.0} \\
 & Classic ToM & 66.7 \std{3.9} & 47.5 \std{2.1} & \textbf{100.0} \std{0.0} & 0.1 \std{0.1} & 67.3 \std{3.8} & 25.6 \std{1.6} & 78.0 \std{2.0} & 24.4 \std{1.3} \\
 & ProAgent & 62.7 \std{4.0} & 51.8 \std{1.8} & \textbf{100.0} \std{0.0} & \textbf{0.0} \std{0.0} & 62.7 \std{4.0} & 29.5 \std{1.4} & 75.1 \std{2.0} & 27.1 \std{1.3} \\
 & HypMinds & 75.3 \std{3.5} & 45.2 \std{1.8} & 96.7 \std{1.5} & 2.2 \std{0.9} & 82.7 \std{3.1} & 24.2 \std{1.5} & 84.9 \std{1.7} & 23.8 \std{1.2} \\
 & \cellcolor{ours}ToS (Ours) & \cellcolor{ours}\textbf{100.0} \std{0.0} & \cellcolor{ours}\textbf{10.0} \std{0.9} & \cellcolor{ours}\textbf{100.0} \std{0.0} & \cellcolor{ours}\textbf{0.0} \std{0.0} & \cellcolor{ours}\textbf{100.0} \std{0.0} & \cellcolor{ours}\textbf{2.3} \std{0.5} & \cellcolor{ours}\textbf{100.0} \std{0.0} & \cellcolor{ours}\textbf{4.1} \std{0.4} \\
\midrule
\multirow{7}{*}{\rotatebox[origin=c]{90}{\textbf{GPT-OSS-20B}}}
 & ReAct & 92.7 \std{2.1} & 19.8 \std{1.5} & 83.3 \std{3.1} & 9.2 \std{1.6} & 77.3 \std{3.4} & 17.7 \std{1.4} & 84.4 \std{1.7} & 15.5 \std{0.9} \\
 & Focal & 72.7 \std{3.7} & 38.0 \std{1.4} & 10.7 \std{2.5} & 57.6 \std{1.9} & 39.3 \std{4.0} & 35.3 \std{1.5} & 40.9 \std{2.3} & 43.6 \std{1.0} \\
 & CFD & 87.3 \std{2.7} & 22.1 \std{1.4} & 56.0 \std{4.1} & 25.5 \std{1.9} & 62.0 \std{4.0} & 23.3 \std{1.3} & 68.4 \std{2.2} & 23.6 \std{0.9} \\
 & Classic ToM & 86.7 \std{2.8} & 25.1 \std{1.8} & 93.3 \std{2.0} & 3.9 \std{1.2} & 68.0 \std{3.8} & 21.8 \std{1.6} & 82.7 \std{1.8} & 16.9 \std{1.0} \\
 & ProAgent & 81.3 \std{3.2} & 34.1 \std{1.7} & 74.0 \std{3.6} & 16.6 \std{2.3} & 55.3 \std{4.1} & 27.7 \std{1.5} & 70.2 \std{2.2} & 26.1 \std{1.1} \\
 & HypMinds & 85.3 \std{2.9} & 30.7 \std{1.7} & 64.0 \std{3.9} & 23.4 \std{2.6} & 40.7 \std{4.0} & 33.7 \std{1.6} & 63.3 \std{2.3} & 29.3 \std{1.2} \\
 & \cellcolor{ours}ToS (Ours) & \cellcolor{ours}\textbf{100.0} \std{0.0} & \cellcolor{ours}\textbf{9.8} \std{1.1} & \cellcolor{ours}\textbf{100.0} \std{0.0} & \cellcolor{ours}\textbf{0.2} \std{0.1} & \cellcolor{ours}\textbf{99.3} \std{0.7} & \cellcolor{ours}\textbf{5.9} \std{0.8} & \cellcolor{ours}\textbf{99.8} \std{0.2} & \cellcolor{ours}\textbf{5.3} \std{0.5} \\
\midrule
\multirow{7}{*}{\rotatebox[origin=c]{90}{\textbf{Qwen3.5-9B}}}
 & ReAct & 45.3 \std{4.1} & 43.1 \std{1.5} & 70.7 \std{3.7} & 18.3 \std{2.2} & 18.0 \std{3.1} & 43.2 \std{1.4} & 44.7 \std{2.3} & 34.9 \std{1.1} \\
 & Focal & 63.3 \std{3.9} & 32.1 \std{1.5} & 34.7 \std{3.9} & 40.8 \std{2.6} & 13.3 \std{2.8} & 50.0 \std{1.5} & 37.1 \std{2.3} & 40.9 \std{1.1} \\
 & CFD & 82.0 \std{3.1} & 25.8 \std{1.4} & 39.3 \std{4.0} & 35.1 \std{2.2} & 12.0 \std{2.7} & 50.7 \std{1.6} & 44.4 \std{2.3} & 37.2 \std{1.1} \\
 & Classic ToM & 60.0 \std{4.0} & 35.5 \std{1.3} & 71.3 \std{3.7} & 18.3 \std{2.3} & 18.7 \std{3.2} & 46.7 \std{1.6} & 50.0 \std{2.4} & 33.5 \std{1.2} \\
 & ProAgent & 75.3 \std{3.5} & 31.8 \std{1.4} & 46.7 \std{4.1} & 36.6 \std{2.9} & 19.3 \std{3.2} & 45.7 \std{1.6} & 47.1 \std{2.4} & 38.0 \std{1.2} \\
 & HypMinds & 82.7 \std{3.1} & 32.0 \std{1.4} & 37.3 \std{4.0} & 45.0 \std{2.9} & 28.0 \std{3.7} & 39.9 \std{1.5} & 49.3 \std{2.4} & 39.0 \std{1.2} \\
 & \cellcolor{ours}ToS (Ours) & \cellcolor{ours}\textbf{98.0} \std{1.1} & \cellcolor{ours}\textbf{24.0} \std{1.3} & \cellcolor{ours}\textbf{100.0} \std{0.0} & \cellcolor{ours}\textbf{0.1} \std{0.1} & \cellcolor{ours}\textbf{96.0} \std{1.6} & \cellcolor{ours}\textbf{14.8} \std{1.0} & \cellcolor{ours}\textbf{98.0} \std{0.7} & \cellcolor{ours}\textbf{13.0} \std{0.7} \\
\bottomrule
\end{tabular}

%% file: 00_results/impact_heterogeneous.tex
\begin{tabular}{c l cc cc cc cc}
\toprule
& & \multicolumn{2}{c}{Competitive} & \multicolumn{2}{c}{Cooperative} & \multicolumn{2}{c}{Mixed} & \multicolumn{2}{c}{Avg.} \\
\cmidrule(lr){3-4}\cmidrule(lr){5-6}\cmidrule(lr){7-8}\cmidrule(lr){9-10}
A1 & Method & Success (\%) $\uparrow$ & Stuck (\%) $\downarrow$ & Success (\%) $\uparrow$ & Stuck (\%) $\downarrow$ & Success (\%) $\uparrow$ & Stuck (\%) $\downarrow$ & Success (\%) $\uparrow$ & Stuck (\%) $\downarrow$ \\
\midrule
\multirow{7}{*}{\rotatebox[origin=c]{90}{\textbf{Gemma-4-26B}}} & ReAct & 53.3 \std{4.1} & 44.2 \std{1.8} & 98.7 \std{0.9} & 1.4 \std{0.5} & 52.7 \std{4.1} & 29.3 \std{1.0} & 68.2 \std{2.2} & 25.0 \std{1.1} \\
 & Focal & 65.3 \std{3.9} & 37.0 \std{1.4} & 58.0 \std{4.0} & 24.7 \std{2.0} & 61.3 \std{4.0} & 27.9 \std{1.0} & 61.6 \std{2.3} & 29.8 \std{0.9} \\
 & CFD & 86.7 \std{2.8} & 19.0 \std{1.3} & 77.3 \std{3.4} & 13.5 \std{1.7} & 82.7 \std{3.1} & 15.1 \std{1.0} & 82.2 \std{1.8} & 15.9 \std{0.8} \\
 & Classic ToM & 78.7 \std{3.4} & 36.4 \std{1.3} & 97.3 \std{1.3} & 1.6 \std{0.6} & 75.3 \std{3.5} & 21.0 \std{1.4} & 83.8 \std{1.7} & 19.7 \std{0.9} \\
 & ProAgent & 84.7 \std{3.0} & 37.1 \std{1.6} & 90.7 \std{2.4} & 6.2 \std{1.1} & 62.7 \std{4.0} & 26.6 \std{1.5} & 79.3 \std{1.9} & 23.3 \std{1.0} \\
 & HypMinds & 82.7 \std{3.1} & 34.9 \std{1.6} & 75.3 \std{3.5} & 16.7 \std{2.3} & 66.0 \std{3.9} & 24.6 \std{1.5} & 74.7 \std{2.1} & 25.4 \std{1.1} \\
 & \cellcolor{ours}ToS (Ours) & \cellcolor{ours}\textbf{100.0} \std{0.0} & \cellcolor{ours}\textbf{10.5} \std{1.0} & \cellcolor{ours}\textbf{100.0} \std{0.0} & \cellcolor{ours}\textbf{0.0} \std{0.0} & \cellcolor{ours}\textbf{100.0} \std{0.0} & \cellcolor{ours}\textbf{2.8} \std{0.5} & \cellcolor{ours}\textbf{100.0} \std{0.0} & \cellcolor{ours}\textbf{4.5} \std{0.4} \\
\midrule
\multirow{7}{*}{\rotatebox[origin=c]{90}{\textbf{Qwen3.5-9B}}} & ReAct & 50.7 \std{4.1} & 40.2 \std{1.4} & 78.0 \std{3.4} & 12.1 \std{1.7} & 28.7 \std{3.7} & 37.1 \std{1.3} & 52.4 \std{2.4} & 29.8 \std{1.0} \\
 & Focal & 58.0 \std{4.0} & 34.4 \std{1.4} & 44.0 \std{4.1} & 39.8 \std{2.7} & 23.3 \std{3.5} & 44.1 \std{1.5} & 41.8 \std{2.3} & 39.4 \std{1.2} \\
 & CFD & 60.7 \std{4.0} & 30.8 \std{1.4} & 49.3 \std{4.1} & 33.2 \std{2.4} & 16.7 \std{3.1} & 47.1 \std{1.6} & 42.2 \std{2.3} & 37.0 \std{1.1} \\
 & Classic ToM & 81.3 \std{3.2} & 28.1 \std{1.4} & 81.3 \std{3.2} & 12.0 \std{1.9} & 36.0 \std{3.9} & 38.1 \std{1.8} & 66.2 \std{2.2} & 26.1 \std{1.1} \\
 & ProAgent & 80.7 \std{3.2} & 31.6 \std{1.4} & 62.7 \std{4.0} & 27.1 \std{2.8} & 31.3 \std{3.8} & 39.5 \std{1.6} & 58.2 \std{2.3} & 32.7 \std{1.2} \\
 & HypMinds & 89.3 \std{2.5} & 29.8 \std{1.5} & 54.7 \std{4.1} & 32.5 \std{3.0} & 40.7 \std{4.0} & 34.9 \std{1.6} & 61.6 \std{2.3} & 32.4 \std{1.2} \\
 & \cellcolor{ours}ToS (Ours) & \cellcolor{ours}\textbf{98.0} \std{1.1} & \cellcolor{ours}\textbf{27.2} \std{1.4} & \cellcolor{ours}\textbf{100.0} \std{0.0} & \cellcolor{ours}\textbf{0.1} \std{0.1} & \cellcolor{ours}\textbf{96.7} \std{1.5} & \cellcolor{ours}\textbf{13.2} \std{1.0} & \cellcolor{ours}\textbf{98.2} \std{0.6} & \cellcolor{ours}\textbf{13.5} \std{0.8} \\
\bottomrule
\end{tabular}

%% file: 00_results/crossplay.tex
\begin{tabular}{c l cc cc cc cc}
\toprule
& & \multicolumn{2}{c}{Competitive} & \multicolumn{2}{c}{Cooperative} & \multicolumn{2}{c}{Mixed} & \multicolumn{2}{c}{Avg.} \\
\cmidrule(lr){3-4}\cmidrule(lr){5-6}\cmidrule(lr){7-8}\cmidrule(lr){9-10}
A1 & A0 Method & Success (\%) $\uparrow$ & Stuck (\%) $\downarrow$ & Success (\%) $\uparrow$ & Stuck (\%) $\downarrow$ & Success (\%) $\uparrow$ & Stuck (\%) $\downarrow$ & Success (\%) $\uparrow$ & Stuck (\%) $\downarrow$ \\
\midrule
\multirow{7}{*}{\rotatebox[origin=c]{90}{\textbf{ReAct}}} & ReAct & 40.7 \std{4.0} & 48.7 \std{1.4} & 96.7 \std{1.5} & 2.5 \std{0.7} & 56.7 \std{4.1} & 26.9 \std{1.1} & 64.7 \std{2.3} & 26.0 \std{1.1} \\
 & Focal & 63.3 \std{3.9} & 38.5 \std{1.6} & 56.0 \std{4.1} & 26.1 \std{2.2} & 50.7 \std{4.1} & 28.3 \std{1.1} & 56.7 \std{2.3} & 30.9 \std{1.0} \\
 & CFD & 60.0 \std{4.0} & 31.4 \std{1.4} & 77.3 \std{3.4} & 14.3 \std{1.9} & 66.7 \std{3.9} & 18.7 \std{1.2} & 68.0 \std{2.2} & 21.5 \std{0.9} \\
 & Classic ToM & 74.0 \std{3.6} & 29.8 \std{1.7} & 94.7 \std{1.8} & 4.5 \std{1.2} & 73.3 \std{3.6} & 21.1 \std{1.3} & 80.7 \std{1.9} & \textbf{18.5} \std{1.0} \\
 & ProAgent & 82.0 \std{3.1} & 31.6 \std{1.6} & 85.3 \std{2.9} & 9.8 \std{1.7} & 66.0 \std{3.9} & 23.8 \std{1.2} & 77.8 \std{2.0} & 21.7 \std{1.0} \\
 & HypMinds & 88.7 \std{2.6} & \textbf{27.0} \std{1.5} & 82.7 \std{3.1} & 10.8 \std{1.8} & 68.0 \std{3.8} & 23.5 \std{1.3} & 79.8 \std{1.9} & 20.4 \std{0.9} \\
 & \cellcolor{ours}ToS (Ours) & \cellcolor{ours}\textbf{100.0} \std{0.0} & \cellcolor{ours}36.2 \std{1.0} & \cellcolor{ours}\textbf{98.0} \std{1.1} & \cellcolor{ours}\textbf{1.7} \std{0.6} & \cellcolor{ours}\textbf{94.0} \std{1.9} & \cellcolor{ours}\textbf{18.1} \std{0.8} & \cellcolor{ours}\textbf{97.3} \std{0.8} & \cellcolor{ours}18.7 \std{0.8} \\
\midrule
\multirow{7}{*}{\rotatebox[origin=c]{90}{\textbf{Classic ToM}}} & ReAct & 97.3 \std{1.3} & 8.2 \std{1.2} & 88.7 \std{2.6} & 6.2 \std{1.3} & 81.3 \std{3.2} & 13.5 \std{1.2} & 89.1 \std{1.5} & 9.3 \std{0.7} \\
 & Focal & 92.0 \std{2.2} & 15.0 \std{1.4} & 51.3 \std{4.1} & 34.1 \std{2.8} & 81.3 \std{3.2} & 15.2 \std{1.3} & 74.9 \std{2.0} & 21.5 \std{1.2} \\
 & CFD & 84.7 \std{3.0} & 16.6 \std{1.5} & 68.0 \std{3.8} & 18.7 \std{2.1} & 72.0 \std{3.7} & 18.5 \std{1.3} & 74.9 \std{2.0} & 17.9 \std{1.0} \\
 & Classic ToM & 72.7 \std{3.7} & 35.6 \std{1.5} & 86.7 \std{2.8} & 8.2 \std{1.6} & 54.0 \std{4.1} & 27.1 \std{1.6} & 71.1 \std{2.1} & 23.6 \std{1.1} \\
 & ProAgent & 83.3 \std{3.1} & 30.2 \std{1.6} & 78.0 \std{3.4} & 14.4 \std{2.1} & 48.0 \std{4.1} & 31.7 \std{1.6} & 69.8 \std{2.2} & 25.4 \std{1.1} \\
 & HypMinds & 75.3 \std{3.5} & 35.9 \std{1.5} & 77.3 \std{3.4} & 14.0 \std{2.1} & 57.3 \std{4.1} & 28.9 \std{1.5} & 70.0 \std{2.2} & 26.3 \std{1.1} \\
 & \cellcolor{ours}ToS (Ours) & \cellcolor{ours}\textbf{100.0} \std{0.0} & \cellcolor{ours}\textbf{7.9} \std{1.0} & \cellcolor{ours}\textbf{98.0} \std{1.1} & \cellcolor{ours}\textbf{1.9} \std{0.5} & \cellcolor{ours}\textbf{97.3} \std{1.3} & \cellcolor{ours}\textbf{11.8} \std{1.0} & \cellcolor{ours}\textbf{98.4} \std{0.6} & \cellcolor{ours}\textbf{7.2} \std{0.5} \\
\bottomrule
\end{tabular}

%% file: 00_results/table_tom_prediction.tex
\begin{tabular}{l cc cc cc}
\toprule
 & \multicolumn{2}{c}{Competitive} & \multicolumn{2}{c}{Cooperative} & \multicolumn{2}{c}{Mixed} \\
\cmidrule(lr){2-3}\cmidrule(lr){4-5}\cmidrule(lr){6-7}
Method & No History & With History & No History & With History & No History & With History \\
\midrule
Classic ToM & 34.3 \std{2.8} & 30.5 \std{1.1} & 91.7 \std{1.6} & 91.5 \std{0.6} & 45.8 \std{3.0} & 56.2 \std{1.1} \\
ProAgent & 29.5 \std{2.6} & 30.3 \std{1.1} & 90.2 \std{1.7} & 78.4 \std{0.9} & 46.3 \std{3.1} & 49.0 \std{1.0} \\
HypMinds & 43.1 \std{2.9} & 29.6 \std{1.1} & 79.7 \std{2.3} & 73.9 \std{0.9} & 48.6 \std{3.0} & 44.4 \std{1.0} \\
\bottomrule
\end{tabular}

%% file: 00_results/table_delta.tex
\begin{tabular}{l cc}
\toprule
 & \multicolumn{2}{c}{\textbf{DivvyBench} - $\delta_{ij}$ at the first step} \\
\cmidrule(lr){2-3}
Method & Competitive $\uparrow$ & Cooperative $\downarrow$ \\
\midrule
ReAct & 0.307 \std{0.043} & 0.007 \std{0.013} \\
Focal & 0.187 \std{0.039} & 0.100 \std{0.031} \\
CFD & 0.480 \std{0.035} & 0.040 \std{0.024} \\
Classic ToM & 0.327 \std{0.041} & 0.047 \std{0.020} \\
ProAgent & 0.280 \std{0.043} & 0.067 \std{0.027} \\
HypMinds & 0.467 \std{0.043} & 0.047 \std{0.026} \\
\cellcolor{ours}ToS (Ours) & \cellcolor{ours}\textbf{0.780} \std{0.034} & \cellcolor{ours}\textbf{0.000} \std{0.000} \\
\bottomrule
\end{tabular}

%% file: 00_results/table_significance.tex
\begin{tabular}{l ccc}
\toprule
 & \multicolumn{3}{c}{\textbf{DivvyBench} - Success (\%) $\uparrow$} \\
\cmidrule(lr){2-4}
Method & Competitive & Cooperative & Mixed \\
\midrule
ReAct & +58.7 {\scriptsize[51.3, 66.0]} \cmark & +3.3 {\scriptsize[0.7, 6.7]} \xmark & +42.7 {\scriptsize[34.7, 50.7]} \cmark \\
Focal & +52.0 {\scriptsize[44.0, 59.3]} \cmark & +58.0 {\scriptsize[52.7, 62.7]} \cmark & +49.3 {\scriptsize[41.3, 57.3]} \cmark \\
CFD & +42.0 {\scriptsize[36.0, 48.0]} \cmark & +33.3 {\scriptsize[31.3, 35.3]} \cmark & +63.3 {\scriptsize[56.7, 70.0]} \cmark \\
Classic ToM & +26.7 {\scriptsize[19.3, 34.0]} \cmark & +13.3 {\scriptsize[8.7, 18.0]} \cmark & +45.3 {\scriptsize[37.3, 53.3]} \cmark \\
ProAgent & +30.7 {\scriptsize[23.3, 38.0]} \cmark & +28.0 {\scriptsize[22.0, 34.0]} \cmark & +64.7 {\scriptsize[57.3, 72.0]} \cmark \\
HypMinds & +33.3 {\scriptsize[26.0, 41.3]} \cmark & +38.0 {\scriptsize[34.0, 42.0]} \cmark & +52.0 {\scriptsize[44.0, 60.0]} \cmark \\
\cellcolor{ours}ToS (Ours) & \cellcolor{ours}99.3 {\scriptsize[98.0, 100.0]} & \cellcolor{ours}100.0 {\scriptsize[100.0, 100.0]} & \cellcolor{ours}99.3 {\scriptsize[98.0, 100.0]} \\
\midrule
 & \multicolumn{2}{c}{\textbf{GovSim}} & \textbf{Overcooked} \\
\cmidrule(lr){2-3}\cmidrule(lr){4-4}
 & Survival (\%) $\uparrow$ & Total Gain $\uparrow$ & $\times$ReAct $\uparrow$ \\
\midrule
ReAct & +58.7 {\scriptsize[53.3, 63.9]} \cmark & +246 {\scriptsize[223, 269]} \cmark & +0.67 {\scriptsize[0.51, 0.94]} \cmark \\
Focal & +65.0 {\scriptsize[59.6, 70.1]} \cmark & +257 {\scriptsize[234, 280]} \cmark & +0.56 {\scriptsize[0.41, 0.75]} \cmark \\
CFD & +26.0 {\scriptsize[19.8, 31.9]} \cmark & +58 {\scriptsize[29, 88]} \cmark & +0.51 {\scriptsize[0.37, 0.68]} \cmark \\
Classic ToM & +53.0 {\scriptsize[47.9, 58.0]} \cmark & +193 {\scriptsize[169, 217]} \cmark & +0.26 {\scriptsize[0.10, 0.45]} \cmark \\
ProAgent & +68.7 {\scriptsize[63.7, 73.5]} \cmark & +284 {\scriptsize[264, 303]} \cmark & +0.31 {\scriptsize[0.16, 0.49]} \cmark \\
HypMinds & +38.7 {\scriptsize[33.3, 44.1]} \cmark & +139 {\scriptsize[115, 164]} \cmark & +0.18 {\scriptsize[0.02, 0.36]} \cmark \\
\cellcolor{ours}ToS (Ours) & \cellcolor{ours}85.2 {\scriptsize[80.9, 89.2]} & \cellcolor{ours}400 {\scriptsize[382, 419]} & \cellcolor{ours}1.67 {\scriptsize[1.48, 2.05]} \\
\bottomrule
\end{tabular}

%% file: 10_prompts/trace_divvybench_competitive.tex
\begin{tcolorbox}[colback=white!98!gray, colframe=black!70, enhanced, boxrule=0.5pt, arc=1pt, left=5pt, right=5pt, top=5pt, bottom=5pt, fonttitle=\bfseries\small, title=DivvyBench Tabletop (Competitive)]
\scriptsize

\textit{Step 1.} Table: \textcolor{red!85!black}{red} \textcolor{blue!80!black}{blue} \textcolor{green!55!black}{green} \textcolor{yellow!65!black}{yellow} \textcolor{orange!90!black}{orange} \textcolor{purple!85!black}{purple} \textcolor{magenta!60!black}{pink} \textcolor{brown!85!black}{brown}
\begin{itemize}[leftmargin=15pt, noitemsep, topsep=0pt, partopsep=0pt, parsep=0pt, after=\vspace{4pt}]
    \item \textbf{A0} pick \textcolor{red!85!black}{red} $\rightarrow$ \cmark
    \item \textbf{A1} pick \textcolor{blue!80!black}{blue} $\rightarrow$ \cmark
\end{itemize}

\textit{Step 2.} Table: \textcolor{green!55!black}{green} \textcolor{yellow!65!black}{yellow} \textcolor{orange!90!black}{orange} \textcolor{purple!85!black}{purple} \textcolor{magenta!60!black}{pink} \textcolor{brown!85!black}{brown}
\begin{itemize}[leftmargin=15pt, noitemsep, topsep=0pt, partopsep=0pt, parsep=0pt, after=\vspace{4pt}]
    \item \textbf{A0} pick \textcolor{green!55!black}{green} $\rightarrow$ \cmark
    \item \textbf{A1} pick \textcolor{yellow!65!black}{yellow} $\rightarrow$ \cmark
\end{itemize}

\centerline{$\dots$}

\textit{Step 4.} Table: \textcolor{magenta!60!black}{pink} \textcolor{brown!85!black}{brown}
\begin{itemize}[leftmargin=15pt, noitemsep, topsep=0pt, partopsep=0pt, parsep=0pt, after=\vspace{4pt}]
    \item \textbf{A0} pick \textcolor{magenta!60!black}{pink} $\rightarrow$ \xmark
    \item \textbf{A1} pick \textcolor{magenta!60!black}{pink} $\rightarrow$ \xmark
\end{itemize}

\textit{Step 5.} Table: \textcolor{magenta!60!black}{pink} \textcolor{brown!85!black}{brown}
\begin{itemize}[leftmargin=15pt, noitemsep, topsep=0pt, partopsep=0pt, parsep=0pt, after=\vspace{4pt}]
    \item \textbf{A0} pick \textcolor{magenta!60!black}{pink} $\rightarrow$ \cmark
    \item \textbf{A1} pick \textcolor{brown!85!black}{brown} $\rightarrow$ \cmark
\end{itemize}

\textit{The table is empty after 5 steps.}
\end{tcolorbox}

%% file: 10_prompts/trace_divvybench_cooperative.tex
\begin{tcolorbox}[colback=white!98!gray, colframe=black!70, enhanced, boxrule=0.5pt, arc=1pt, left=5pt, right=5pt, top=5pt, bottom=5pt, fonttitle=\bfseries\small, title=DivvyBench Tabletop (Cooperative)]
\scriptsize

\textit{Step 1.} Table: \textcolor{red!85!black}{red} \textcolor{blue!80!black}{blue} \textcolor{green!55!black}{green} \textcolor{yellow!65!black}{yellow} \textcolor{orange!90!black}{orange} \textcolor{purple!85!black}{purple} \textcolor{magenta!60!black}{pink} \textcolor{brown!85!black}{brown}
\begin{itemize}[leftmargin=15pt, noitemsep, topsep=0pt, partopsep=0pt, parsep=0pt, after=\vspace{4pt}]
    \item \textbf{A0} pick \textcolor{red!85!black}{red} $\rightarrow$ \cmark
    \item \textbf{A1} pick \textcolor{red!85!black}{red} $\rightarrow$ \cmark
\end{itemize}

\textit{Step 2.} Table: \textcolor{blue!80!black}{blue} \textcolor{green!55!black}{green} \textcolor{yellow!65!black}{yellow} \textcolor{orange!90!black}{orange} \textcolor{purple!85!black}{purple} \textcolor{magenta!60!black}{pink} \textcolor{brown!85!black}{brown}
\begin{itemize}[leftmargin=15pt, noitemsep, topsep=0pt, partopsep=0pt, parsep=0pt, after=\vspace{4pt}]
    \item \textbf{A0} pick \textcolor{blue!80!black}{blue} $\rightarrow$ \cmark
    \item \textbf{A1} pick \textcolor{blue!80!black}{blue} $\rightarrow$ \cmark
\end{itemize}

\centerline{$\dots$}

\textit{Step 7.} Table: \textcolor{magenta!60!black}{pink} \textcolor{brown!85!black}{brown}
\begin{itemize}[leftmargin=15pt, noitemsep, topsep=0pt, partopsep=0pt, parsep=0pt, after=\vspace{4pt}]
    \item \textbf{A0} pick \textcolor{magenta!60!black}{pink} $\rightarrow$ \cmark
    \item \textbf{A1} pick \textcolor{magenta!60!black}{pink} $\rightarrow$ \cmark
\end{itemize}

\textit{Step 8.} Table: \textcolor{brown!85!black}{brown}
\begin{itemize}[leftmargin=15pt, noitemsep, topsep=0pt, partopsep=0pt, parsep=0pt, after=\vspace{4pt}]
    \item \textbf{A0} pick \textcolor{brown!85!black}{brown} $\rightarrow$ \cmark
    \item \textbf{A1} pick \textcolor{brown!85!black}{brown} $\rightarrow$ \cmark
\end{itemize}

\textit{The table is empty after 8 steps.}
\end{tcolorbox}

%% file: 10_prompts/prompt_divvybench_competitive.tex
\begin{tcolorbox}[colback=white!98!gray, colframe=black!70, enhanced, boxrule=0.5pt, arc=1pt, left=6pt, right=6pt, top=6pt, bottom=6pt, fonttitle=\bfseries\small, title=Task Prompt: DivvyBench Tabletop (Competitive)]
\vspace{-10pt}
\begin{lstlisting}[language=systemprompt, escapeinside={(*}{*)}]
There are (*\pinput{n\_balls}*) balls on a table: (*\pinput{ball\_list}*). You and (*\pinput{n\_teammates}*) teammate(s) form a (*\pinput{n\_agents}*)-robot team, and together you must collect all of them in as few steps as possible.

**Rules**
- Each step, each of you reaches for one ball **at the same time** - there is no turn-taking. You may also wait.
- You **cannot** communicate, and you cannot see which ball anyone else will reach for this step.
- If **two or more** of you reach for the same ball, you collide: none of you gets it and the step is wasted.
- A ball is single-use: if exactly one of you reaches for it, that robot collects it for the team; if two or more of you reach for the same ball you collide and nobody gets it (the step is wasted). Keep going until the table is empty.

## Who you are
You are the **(*\pinput{role}*)** of this two-robot team (agent A(*\pinput{agent\_id}*)); your teammate is the **(*\pinput{other\_role}*)** (agent A(*\pinput{other\_id}*)). You both reason in exactly the same way -- your role label is the only thing that distinguishes you.
\end{lstlisting}
\vspace{-10pt}
\end{tcolorbox}

%% file: 10_prompts/prompt_divvybench_cooperative.tex
\begin{tcolorbox}[colback=white!98!gray, colframe=black!70, enhanced, boxrule=0.5pt, arc=1pt, left=6pt, right=6pt, top=6pt, bottom=6pt, fonttitle=\bfseries\small, title=Task Prompt: DivvyBench Tabletop (Cooperative)]
\vspace{-10pt}
\begin{lstlisting}[language=systemprompt, escapeinside={(*}{*)}]
There are (*\pinput{n\_balls}*) balls on a table: (*\pinput{ball\_list}*). You and (*\pinput{n\_teammates}*) teammate(s) form a (*\pinput{n\_agents}*)-robot team, and together you must collect all of them in as few steps as possible.

**Rules**
- Each ball is **heavy**: collecting it needs all of you to lift it together.
- Each step, each of you reaches for one ball **at the same time** - there is no turn-taking. You may also wait.
- You **cannot** communicate, and you cannot see which ball anyone else will reach for this step.
- If **all of you** reach for the same ball, you lift it together and collect it.
- If you do **not all** reach for the same ball (someone reaches elsewhere or waits), nothing is collected and the step is wasted.

## Who you are
You are the **(*\pinput{role}*)** of this two-robot team (agent A(*\pinput{agent\_id}*)); your teammate is the **(*\pinput{other\_role}*)** (agent A(*\pinput{other\_id}*)). You both reason in exactly the same way -- your role label is the only thing that distinguishes you.
\end{lstlisting}
\vspace{-10pt}
\end{tcolorbox}

%% file: 10_prompts/prompt_divvybench_mixed.tex
\begin{tcolorbox}[colback=white!98!gray, colframe=black!70, enhanced, boxrule=0.5pt, arc=1pt, left=6pt, right=6pt, top=6pt, bottom=6pt, fonttitle=\bfseries\small, title=Task Prompt: DivvyBench Tabletop (Mixed)]
\vspace{-10pt}
\begin{lstlisting}[language=systemprompt, escapeinside={(*}{*)}]
There are (*\pinput{n\_balls}*) balls on a table: (*\pinput{ball\_list}*). You and (*\pinput{n\_teammates}*) teammate(s) form a (*\pinput{n\_agents}*)-robot team, and together you must collect all of them in as few steps as possible.

**Rules**
- Each step, each of you reaches for one ball **at the same time** - no turn-taking, no communication, and you cannot see anyone else's choice. You may also wait.
- Balls come in two kinds (each ball is labelled in the state below):
  - a **light** ball is single-use: if **exactly one** of you reaches for it, that robot collects it; if two or more of you reach for the same light ball you collide and nobody gets it (wasted).
  - a **heavy** ball needs all of you: it is collected only if **all of you** reach for it together; anything short of everyone (a split or a partial reach) collects nothing (wasted).
- Keep going until the table is empty.

## Who you are
You are the **(*\pinput{role}*)** of this two-robot team (agent A(*\pinput{agent\_id}*)); your teammate is the **(*\pinput{other\_role}*)** (agent A(*\pinput{other\_id}*)). You both reason in exactly the same way -- your role label is the only thing that distinguishes you.
\end{lstlisting}
\vspace{-10pt}
\end{tcolorbox}

%% file: 10_prompts/prompt_divvybench_airspace.tex
\begin{tcolorbox}[colback=white!98!gray, colframe=black!70, enhanced, boxrule=0.5pt, arc=1pt, left=6pt, right=6pt, top=6pt, bottom=6pt, fonttitle=\bfseries\small, title=Task Prompt: DivvyBench Airspace (Competitive)]
\vspace{-10pt}
\begin{lstlisting}[language=systemprompt, escapeinside={(*}{*)}]
There are (*\pinput{n\_areas}*) areas in a survey grid: (*\pinput{area\_list}*). You and (*\pinput{n\_teammates}*) teammate(s) form a (*\pinput{n\_agents}*)-drone team, and together you must survey all of them in as few steps as possible.

**Rules**
- Each step, each of you flies to one area **at the same time** - there is no turn-taking. You may also wait.
- You **cannot** communicate, and you cannot see which area anyone else will fly to this step.
- If **two or more** of you fly to the same area, you collide: none of you gets it and the step is wasted.
- An area is single-pass: if exactly one of you flies to it, that drone surveys it for the team; if two or more of you fly to the same area you collide and nobody gets it (the step is wasted). Keep going until every area is surveyed.

## Who you are
You are the **(*\pinput{role}*)** of this two-drone team (agent A(*\pinput{agent\_id}*)); your teammate is the **(*\pinput{other\_role}*)** (agent A(*\pinput{other\_id}*)). You both reason in exactly the same way -- your role label is the only thing that distinguishes you.
\end{lstlisting}
\vspace{-10pt}
\end{tcolorbox}

%% file: 10_prompts/prompt_divvybench_household.tex
\begin{tcolorbox}[colback=white!98!gray, colframe=black!70, enhanced, boxrule=0.5pt, arc=1pt, left=6pt, right=6pt, top=6pt, bottom=6pt, fonttitle=\bfseries\small, title=Task Prompt: DivvyBench Household (Competitive)]
\vspace{-10pt}
\begin{lstlisting}[language=systemprompt, escapeinside={(*}{*)}]
There are (*\pinput{n\_locations}*) locations in an apartment: (*\pinput{location\_list}*). You and (*\pinput{n\_teammates}*) teammate(s) form a (*\pinput{n\_agents}*)-robot team, and together you must search all of them in as few steps as possible.

**Rules**
- Each step, each of you goes to one location **at the same time** - there is no turn-taking. You may also wait.
- You **cannot** communicate, and you cannot see which location anyone else will go to this step.
- If **two or more** of you go to the same location, you collide: none of you gets it and the step is wasted.
- A location is single-visit: if exactly one of you goes to it, that robot searches it for the team; if two or more of you go to the same location you collide and nobody gets it (the step is wasted). Keep going until every location is searched.

## Who you are
You are the **(*\pinput{role}*)** of this two-robot team (agent A(*\pinput{agent\_id}*)); your teammate is the **(*\pinput{other\_role}*)** (agent A(*\pinput{other\_id}*)). You both reason in exactly the same way -- your role label is the only thing that distinguishes you.
\end{lstlisting}
\vspace{-10pt}
\end{tcolorbox}

%% file: 10_prompts/prompt_tos.tex
\begin{tcolorbox}[colback=white!98!gray, colframe=black!70, enhanced, boxrule=0.5pt, arc=1pt, left=6pt, right=6pt, top=6pt, bottom=6pt, fonttitle=\bfseries\small, title=Prompt Template: Theory of Scene (ToS)]
\vspace{-10pt}
\begin{lstlisting}[language=systemprompt, escapeinside={(*}{*)}]
(*\pinput{benchmark\_scaffold}*)

## Coordinating from the shared observation
Decide from the observation you and the other agents all see: reading the same scene and reasoning alike, a rule anchored to that scene leads you all to the same division of labor.

**Role.** Read the ownership: is it overlapping -- do you and the other agents all work the same targets and resources -- or is it divided, each of you already holding its own part?
  - **Overlapping** (acting alike you would contend for one target or all defer): coordinate -- settle a division of labor over the shared targets from the scene.
  - **Divided** (your role owns some stages/stations and the other agents own the rest): hold your part and execute. Take the single most useful action within your own part; do not step onto a station or task that belongs to another agent's part, even if it looks like the most useful step right now -- reading the same scene, they are already taking it, so you would only collide. If your own part has no ready step (it waits on their output), do its enabling step or start the next independent unit; wait only when nothing of yours is productive.

**Task.** Settle the division over the shared targets: read how each pending target couples you and apply its rule over the canonical order:
  1. **Joint** (succeeds only if you all act on it together): converge -- all take the same one: the first such target in that order. (Acting alone wastes it.)
  2. **Exclusive** (single-occupancy -- multiple acting on it wastes it): divide by item -- take your share by role over that order; the other agents take the complement, so you never collide.
  3. **Sequential** (ordered hand-off chain): divide by stage -- work the stage your role owns, by the chain order.

## Reply format
Reply with a single JSON object and nothing else:

{
  "state":   "<from the observation only: what you hold + which steps are doable now (inputs present); world state, never the other agents' intent>",
  "role":    "<is ownership overlapping (you and the other agents all work the same targets/resources -> coordinate) or divided (your role owns some stages/stations, the others own the rest -> execute your own), and why>",
  "task":    "<for each pending shared target, how it couples you -- Joint / Exclusive / Sequential>",
  "plan":    "<overlapping: apply the matching rule per target over the canonical order -- Joint: converge on the first such target; Exclusive: your role's share; Sequential: your stage; divided: the most useful step within your own part (enabling step or next unit if none ready), not a step that belongs to another agent's part>",
  "action":  "<copy one action verbatim from your currently-legal actions>"
}
\end{lstlisting}
\vspace{-10pt}
\end{tcolorbox}
  

%% file: 10_prompts/trace_tos_divvybench.tex
\begin{tcolorbox}[colback=white!98!gray, colframe=black!70, enhanced, boxrule=0.5pt, arc=1pt, left=6pt, right=6pt, top=6pt, bottom=6pt, fonttitle=\bfseries\small, title=ToS Reasoning on DivvyBench Tabletop (Mixed)]
\scriptsize

\textbf{Context.} Agents: 2 \quad Balls: 8, 4 heavy and 4 light \quad Rule: a heavy ball needs both agents at once, a light ball needs exactly one and collides if both take it

\medskip\hrule\medskip

\textit{Step 1.} On the table: red (\textit{heavy}), blue (\textit{heavy}), green (\textit{heavy}), yellow (\textit{heavy}), orange (\textit{light}), purple (\textit{light}), pink (\textit{light}), brown (\textit{light}).

\begin{minipage}[t]{0.478\linewidth}
\textbf{A0} $\rightarrow$ \texttt{pick\_red}
\begin{itemize}[leftmargin=10pt, noitemsep, topsep=1pt, partopsep=0pt, parsep=0pt]
    \item \texttt{role}: \reading{Overlapping}. Both agents have access to the same set of targets and resources; there are no distinct stages or stations assigned to specific roles. Therefore, we must [...]
    \item \texttt{task}: \reading{Joint}: red, blue, green, yellow (require both agents to pick the same ball). \reading{Exclusive}: orange, purple, pink, brown (require exactly one agent to pick; if both pick, [...]
    \item \texttt{plan}: Apply the canonical order rule. For Joint targets, converge on the first one in the list (red). For Exclusive targets, split them by role (Leader takes the first [...]
\end{itemize}
\end{minipage}\hfill
\begin{minipage}[t]{0.478\linewidth}
\textbf{A1} $\rightarrow$ \texttt{pick\_red}
\begin{itemize}[leftmargin=10pt, noitemsep, topsep=1pt, partopsep=0pt, parsep=0pt]
    \item \texttt{role}: \reading{Overlapping}. Both agents see the same state, have the same capabilities, and no role-specific stations or stages exist. Therefore, they must coordinate to avoid [...]
    \item \texttt{task}: \reading{Joint}: The four heavy balls (red, blue, green, yellow) require both agents to pick the same one simultaneously to succeed. \reading{Exclusive}: The four light balls (orange, [...]
    \item \texttt{plan}: Apply the coordination rule: 1. Joint targets: Both agents must converge on the first heavy ball in the canonical order (red). 2. Exclusive targets: Since we are [...]
\end{itemize}
\end{minipage}

\medskip\hrule\medskip

\textit{Step 5.} On the table: orange (\textit{light}), purple (\textit{light}), pink (\textit{light}), brown (\textit{light}).

\begin{minipage}[t]{0.478\linewidth}
\textbf{A0} $\rightarrow$ \texttt{pick\_orange}
\begin{itemize}[leftmargin=10pt, noitemsep, topsep=1pt, partopsep=0pt, parsep=0pt]
    \item \texttt{role}: \reading{Overlapping}. Both agents are identical in capability and have access to the same set of remaining targets (the four light balls). There are no distinct stages or [...]
    \item \texttt{task}: All remaining targets (orange, purple, pink, brown) are \reading{Exclusive} (light balls). If both agents pick the same one, they collide and fail; if they pick different ones, [...]
    \item \texttt{plan}: The agents must divide the Exclusive targets to avoid collision. Using the canonical order (orange, purple, pink, brown) and a deterministic role-based split (e.g., [...]
\end{itemize}
\end{minipage}\hfill
\begin{minipage}[t]{0.478\linewidth}
\textbf{A1} $\rightarrow$ \texttt{pick\_purple}
\begin{itemize}[leftmargin=10pt, noitemsep, topsep=1pt, partopsep=0pt, parsep=0pt]
    \item \texttt{role}: \reading{Overlapping}
    \item \texttt{task}: \reading{Exclusive}: The remaining targets are all light (single-use) balls. Since both robots are identical and overlapping, acting on the same ball causes a collision (waste). [...]
    \item \texttt{plan}: As the follower (A1) and the leader (A0) share the same reasoning and canonical order, we must split the Exclusive targets. A standard coordination rule for overlapping [...]
\end{itemize}
\end{minipage}

\end{tcolorbox}

%% file: 10_prompts/trace_tos_govsim.tex
\begin{tcolorbox}[colback=white!98!gray, colframe=black!70, enhanced, boxrule=0.5pt, arc=1pt, left=6pt, right=6pt, top=6pt, bottom=6pt, fonttitle=\bfseries\small, title=ToS Reasoning on GovSim Fishery]
\scriptsize

\textbf{Context.} Agents: 5 \quad Lake: 100 tons \quad Regrowth: survivors double, capped at 100 \quad Horizon: 12 months

\medskip\hrule\medskip

\textit{Month 1 of 12.} Lake at 100 tons.

\textbf{A0} $\rightarrow$ \texttt{harvest 10}
\begin{itemize}[leftmargin=15pt, noitemsep, topsep=1pt, partopsep=0pt, parsep=0pt]
    \item \texttt{role}: \reading{Overlapping}; all 5 agents are fishermen with identical capabilities and access to the same shared resource (the lake), contending for the same fish stock.
    \item \texttt{task}: \reading{Exclusive}; the fish in the lake are a shared resource where multiple agents acting on it (catching) creates competition. To avoid overfishing and ensure sustainability for long-term income, the agents must divide the catch among themselves.
    \item \texttt{plan}: Since we are overlapping and the resource is Exclusive (single-occupancy per unit of fish, but shared by the group), we must coordinate a division of labor to maximize long-term yield. The optimal sustainable strategy for a logistic growth model with carrying capacity K=100 and growth rate r=1 (doubling) is to harvest at the Maximum [...]
\end{itemize}

\textbf{A1} $\rightarrow$ \texttt{harvest 10}\\
\textbf{A2} $\rightarrow$ \texttt{harvest 10}\\
\textbf{A3} $\rightarrow$ \texttt{harvest 10}\\
\textbf{A4} $\rightarrow$ \texttt{harvest 10}

\end{tcolorbox}

%% file: 10_prompts/trace_tos_overcooked.tex
\begin{tcolorbox}[colback=white!98!gray, colframe=black!70, enhanced, boxrule=0.5pt, arc=1pt, left=6pt, right=6pt, top=6pt, bottom=6pt, fonttitle=\bfseries\small, title=ToS Reasoning on Overcooked]
\scriptsize

\textbf{Context.} Agents: 2 \quad Dish: baked bell pepper soup (level 3) \quad Recipe: bell pepper $\to$ cut $\to$ bake $\to$ cook $\to$ plate $\to$ deliver \quad Stations: A0 pot (cook) and oven (bake), A1 chopping board (cut) and blender (stir)

\medskip\hrule\medskip

\textit{Step 1.} Kitchen empty, both agents idle.

\begin{minipage}[t]{0.478\linewidth}
\textbf{A0} $\rightarrow$ \texttt{pickup(dish,dish\_dispenser)}
\begin{itemize}[leftmargin=10pt, noitemsep, topsep=1pt, partopsep=0pt, parsep=0pt]
    \item \texttt{role}: \reading{Divided}. Chef owns cooking stations (pot, oven) and serving, while Assistant owns prep stations (ingredient dispenser, chopping board, blender). Since the first step [...]
    \item \texttt{task}: \reading{Sequential}. The process is a chain: Ingredient Pickup $\to$ Cut $\to$ Bake $\to$ Cook $\to$ Serve. The Chef's stage is the second half (Bake, Cook, Serve). The Assistant's stage is [...]
    \item \texttt{plan}: Divided roles: Assistant handles ingredient pickup and chopping. Chef handles baking, cooking, and serving. Since the Assistant must start the process by picking up the [...]
\end{itemize}
\end{minipage}\hfill
\begin{minipage}[t]{0.478\linewidth}
\textbf{A1} $\rightarrow$ \texttt{pickup(bell\_pepper,ingredient\_dispenser)}
\begin{itemize}[leftmargin=10pt, noitemsep, topsep=1pt, partopsep=0pt, parsep=0pt]
    \item \texttt{role}: \reading{Divided}
    \item \texttt{task}: \reading{Sequential}
    \item \texttt{plan}: My role is the Assistant, owning the prep stage (ingredient dispenser, chopping board). The Chef owns the cooking stage (oven, pot). The workflow is sequential: I must [...]
\end{itemize}
\end{minipage}

\medskip\hrule\medskip

\textit{Step 2.} A0 holds a dish, A1 holds a bell pepper, all stations empty.

\begin{minipage}[t]{0.478\linewidth}
\textbf{A0} $\rightarrow$ \texttt{wait(1)}
\begin{itemize}[leftmargin=10pt, noitemsep, topsep=1pt, partopsep=0pt, parsep=0pt]
    \item \texttt{role}: \reading{Divided}. The Chef is responsible for cooking stations (oven, pot) and serving, while the Assistant handles prep (chopping board) and fetching. The Chef cannot cut [...]
    \item \texttt{task}: \reading{Sequential}. The workflow requires: 1. Assistant cuts pepper $\to$ 2. Hand off slices $\to$ 3. Chef bakes $\to$ 4. Hand off baked slices $\to$ 5. Chef cooks $\to$ 6. Chef plates. The [...]
    \item \texttt{plan}: The Assistant currently holds the raw bell pepper and needs to cut it. The Chef holds a dish (prepared for serving) but cannot cook anything yet because the ingredients [...]
\end{itemize}
\end{minipage}\hfill
\begin{minipage}[t]{0.478\linewidth}
\textbf{A1} $\rightarrow$ \texttt{put\_obj\_in\_utensil(chopping\_board0)}
\begin{itemize}[leftmargin=10pt, noitemsep, topsep=1pt, partopsep=0pt, parsep=0pt]
    \item \texttt{role}: \reading{Divided}. The Assistant owns the prep stage (ingredient dispenser, chopping\_board, blender) and the Chef owns the cooking stage (oven, pot) and serving. The workflow is [...]
    \item \texttt{task}: \reading{Sequential}. The target is the 'Baked Bell Pepper Soup' order. The chain is: 1. Assistant cuts bell\_pepper (my stage). 2. Chef bakes slices (their stage). 3. Chef cooks [...]
    \item \texttt{plan}: I have the bell\_pepper in hand and the chopping\_board0 is empty. According to the recipe, I must cut the bell pepper. The legal action to start this process is to place [...]
\end{itemize}
\end{minipage}

\end{tcolorbox}

%% file: references.bib
@InProceedings{goel2025great,
  title={{Great Models Think Alike and this Undermines {AI} Oversight}},
  author={Goel, Shashwat and Str\"{u}ber, Joschka and Auzina, Ilze Amanda and Chandra, Karuna K and Kumaraguru, Ponnurangam and Kiela, Douwe and Prabhu, Ameya and Bethge, Matthias and Geiping, Jonas},
  booktitle={International Conference on Machine Learning},
  pages={19621--19678},
  year={2025}
}

@inproceedings{kim2025correlated,
  title={{Correlated Errors in Large Language Models}},
  author={Kim, Elliot and Garg, Avi and Peng, Kenny and Garg, Nikhil},
  booktitle={International Conference on Machine Learning},
  pages={30038--30066},
  year={2025}
}

@article{jiang2025artificial,
  title={{Artificial Hivemind: The Open-Ended Homogeneity of Language Models (and Beyond)}},
  author={Jiang, Liwei and Chai, Yuanjun and Li, Margaret and Liu, Mickel and Fok, Raymond and Dziri, Nouha and Tsvetkov, Yulia and Sap, Maarten and Choi, Yejin},
  journal={Advances in Neural Information Processing Systems},
  volume={38},
  year={2025}
}

@article{yang2025agentnet,
  title={{AgentNet: Decentralized Evolutionary Coordination for LLM-based Multi-Agent Systems}},
  author={Yang, Yingxuan and Chai, Huacan and Shao, Shuai and Song, Yuanyi and Qi, Siyuan and Rui, Renting and Zhang, Weinan},
  journal={Advances in Neural Information Processing Systems},
  volume={38},
  pages={107309--107336},
  year={2025}
}

@inproceedings{mu2026adaptive,
  title={{Adaptive Theory of Mind for LLM-based Multi-Agent Coordination}},
  author={Mu, Chunjiang and Zeng, Ya and Zhang, Qiaosheng and Shao, Kun and Chu, Chen and Guo, Hao and Jia, Danyang and Wang, Zhen and Hu, Shuyue},
  booktitle={Proceedings of the AAAI Conference on Artificial Intelligence},
  volume={40},
  number={35},
  pages={29608--29616},
  year={2026}
}

@article{kim2026teambench,
  title={{TeamBench: Evaluating Agent Coordination under Enforced Role Separation}},
  author={Kim, Yubin and Park, Chanwoo and Kim, Taehan and Park, Eugene and Schmidgall, Samuel and Rahman, Salman and Park, Chunjong and Breazeal, Cynthia and Liu, Xin and Palangi, Hamid and others},
  journal={arXiv preprint arXiv:2605.07073},
  year={2026}
}

@article{tessera2026benchmarking,
  title={{Benchmarking Open-Ended Multi-Agent Coordination in Language Agents}},
  author={Tessera, Kale-ab Abebe and Szecsenyi, Andras and Barker, Cameron and Rutherford, Alexander and Paglieri, Davide and Scannell, Aidan and Gouk, Henry and Crowley, Elliot J and Rockt{\"a}schel, Tim and Storkey, Amos},
  journal={arXiv preprint arXiv:2606.08340},
  year={2026}
}

@inproceedings{wu2024autogen,
  title={{AutoGen: Enabling Next-Gen LLM Applications via Multi-Agent Conversation}},
  author={Wu, Qingyun and Bansal, Gagan and Zhang, Jieyu and Wu, Yiran and Li, Beibin and Zhu, Erkang and Jiang, Li and Zhang, Xiaoyun and Zhang, Shaokun and Liu, Jiale and others},
  booktitle={First Conference on Language Modeling},
  year={2024}
}

@article{li2023camel,
  title={{CAMEL: Communicative Agents for "Mind" Exploration of Large Language Model Society}},
  author={Li, Guohao and Hammoud, Hasan and Itani, Hani and Khizbullin, Dmitrii and Ghanem, Bernard},
  journal={Advances in Neural Information Processing Systems},
  volume={36},
  pages={51991--52008},
  year={2023}
}

@inproceedings{hong2024metagpt,
  title={{MetaGPT: Meta Programming for A Multi-Agent Collaborative Framework}},
  author={Hong, Sirui and Zhuge, Mingchen and Chen, Jonathan and Zheng, Xiawu and Cheng, Yuheng and Wang, Jinlin and Zhang, Ceyao and Yau, Steven and Lin, Zijuan and Zhou, Liyang and others},
  booktitle={International Conference on Learning Representations},
  pages={23247--23275},
  year={2024}
}

@inproceedings{zhang2024building,
  title={{Building Cooperative Embodied Agents Modularly with Large Language Models}},
  author={Zhang, Hongxin and Du, Weihua and Shan, Jiaming and Zhou, Qinhong and Du, Yilun and Tenenbaum, Joshua B and Shu, Tianmin and Gan, Chuang},
  booktitle={International Conference on Learning Representations},
  pages={19373--19401},
  year={2024}
}

@inproceedings{agashe2025llm,
  title={{LLM-Coordination: Evaluating and Analyzing Multi-agent Coordination Abilities in Large Language Models}},
  author={Agashe, Saaket and Fan, Yue and Reyna, Anthony and Wang, Xin Eric},
  booktitle={Findings of the Association for Computational Linguistics: NAACL 2025},
  pages={8038--8057},
  year={2025}
}

@inproceedings{cross2025hypothetical,
  title={{Hypothetical Minds: Scaffolding Theory of Mind for Multi-Agent Tasks with Large Language Models}},
  author={Cross, Logan and Xiang, Violet and Bhatia, Agam and Yamins, Daniel and Haber, Nick},
  booktitle={International Conference on Learning Representations},
  pages={6507--6546},
  year={2025}
}

@inproceedings{zhang2024proagent,
  title={{ProAgent: Building Proactive Cooperative Agents with Large Language Models}},
  author={Zhang, Ceyao and Yang, Kaijie and Hu, Siyi and Wang, Zihao and Li, Guanghe and Sun, Yihang and Zhang, Cheng and Zhang, Zhaowei and Liu, Anji and Zhu, Song-Chun and others},
  booktitle={Proceedings of the AAAI Conference on Artificial Intelligence},
  volume={38},
  number={16},
  pages={17591--17599},
  year={2024}
}

@article{carroll2019utility,
  title={{On the Utility of Learning about Humans for Human-AI Coordination}},
  author={Carroll, Micah and Shah, Rohin and Ho, Mark K and Griffiths, Tom and Seshia, Sanjit and Abbeel, Pieter and Dragan, Anca},
  journal={Advances in Neural Information Processing Systems},
  volume={32},
  year={2019}
}

@inproceedings{sun2025collab,
  title={{Collab-Overcooked: Benchmarking and Evaluating Large Language Models as Collaborative Agents}},
  author={Sun, Haochen and Zhang, Shuwen and Niu, Lujie and Ren, Lei and Xu, Hao and Fu, Hao and Zhao, Fangkun and Yuan, Caixia and Wang, Xiaojie},
  booktitle={Proceedings of the 2025 Conference on Empirical Methods in Natural Language Processing},
  pages={4922--4951},
  year={2025}
}

@inproceedings{qian2024chatdev,
  title={{ChatDev: Communicative Agents for Software Development}},
  author={Qian, Chen and Liu, Wei and Liu, Hongzhang and Chen, Nuo and Dang, Yufan and Li, Jiahao and Yang, Cheng and Chen, Weize and Su, Yusheng and Cong, Xin and others},
  booktitle={Proceedings of the 62nd Annual Meeting of the Association for Computational Linguistics (Volume 1: Long Papers)},
  pages={15174--15186},
  year={2024}
}

@inproceedings{li2023theory,
  title={{Theory of Mind for Multi-Agent Collaboration via Large Language Models}},
  author={Li, Huao and Chong, Yu and Stepputtis, Simon and Campbell, Joseph P and Hughes, Dana and Lewis, Charles and Sycara, Katia},
  booktitle={Proceedings of the 2023 Conference on Empirical Methods in Natural Language Processing},
  pages={180--192},
  year={2023}
}

@inproceedings{hu2020other,
  title={{"Other-Play" for Zero-Shot Coordination}},
  author={Hu, Hengyuan and Lerer, Adam and Peysakhovich, Alex and Foerster, Jakob},
  booktitle={International Conference on Machine Learning},
  pages={4399--4410},
  year={2020},
  organization={PMLR}
}

@inproceedings{yao2023react,
  title={{ReAct: Synergizing Reasoning and Acting in Language Models}},
  author={Yao, Shunyu and Zhao, Jeffrey and Yu, Dian and Du, Nan and Shafran, Izhak and Narasimhan, Karthik and Cao, Yuan},
  booktitle={International Conference on Learning Representations},
  year={2023}
}

@inproceedings{hayler2026zero,
  title={{Zero-Shot Coordination among LLM Agents}},
  author={Hayler, Adrian and Chirra, Shashank Reddy and Lupu, Andrei and Forkel, Johannes and Sarkar, Bidipta and Feng, Siheng and Foerster, Jakob Nicolaus},
  booktitle={Workshop on Multi-Agent Learning and Its Opportunities in the Era of Generative AI},
  year={2026}
}

@article{chopra2025ripple,
  title={{Ripple Effect Protocol: Coordinating Agent Populations}},
  author={Chopra, Ayush and Sharma, Aman and Ahmad, Feroz and Muscariello, Luca and Pandey, Vijoy and Raskar, Ramesh},
  journal={arXiv preprint arXiv:2510.16572},
  year={2025}
}

@article{grotschla2025agentsnet,
  title={{AgentsNet: Coordination and Collaborative Reasoning in Multi-Agent LLMs}},
  author={Gr{\"o}tschla, Florian and M{\"u}ller, Luis and T{\"o}nshoff, Jan and Galkin, Mikhail and Perozzi, Bryan},
  journal={arXiv preprint arXiv:2507.08616},
  year={2025}
}

@inproceedings{wang2022tom2c,
  title={{ToM2C: Target-oriented Multi-agent Communication and Cooperation with Theory of Mind}},
  author={Wang, Yuanfei and Zhong, Fangwei and Xu, Jing and Wang, Yizhou},
  booktitle={International Conference on Learning Representations},
  year={2022}
}

@article{jian2026gated,
  title={{Gated Coordination for Efficient Multi-Agent Collaboration in Minecraft Game}},
  author={Jian, HuaDong and Li, Chenghao and Wang, Haoyu and Shuai, Jiajia and Guo, Jinyu and Yang, Yang and Zhang, Chaoning},
  journal={arXiv preprint arXiv:2604.18975},
  year={2026}
}

@article{piatti2024cooperate,
  title={{Cooperate or Collapse: Emergence of Sustainable Cooperation in a Society of LLM Agents}},
  author={Piatti, Giorgio and Jin, Zhijing and Kleiman-Weiner, Max and Sch{\"o}lkopf, Bernhard and Sachan, Mrinmaya and Mihalcea, Rada},
  journal={Advances in Neural Information Processing Systems},
  volume={37},
  pages={111715--111759},
  year={2024}
}

@misc{qwen3.5,
  title= {{Qwen3.5: Towards Native Multimodal Agents}},
  author={{Qwen Team}},
  year={2026},
  month={February},
  url={https://qwen.ai/blog?id=qwen3.5}
}

@article{lupu2025decrypto,
  title={{The Decrypto Benchmark for Multi-Agent Reasoning and Theory of Mind}},
  author={Lupu, Andrei and Willi, Timon and Foerster, Jakob},
  journal={arXiv preprint arXiv:2506.20664},
  year={2025}
}

@article{li2025systematic,
  title={{Systematic Failures in Collective Reasoning under Distributed Information in Multi-Agent LLMs}},
  author={Li, Yuxuan and Naito, Aoi and Shirado, Hirokazu},
  journal={arXiv preprint arXiv:2505.11556},
  year={2025}
}

@article{abdelnabi2024cooperation,
  title={{Cooperation, Competition, and Maliciousness: LLM-Stakeholders Interactive Negotiation}},
  author={Abdelnabi, Sahar and Gomaa, Amr and Sivaprasad, Sarath and Sch{\"o}nherr, Lea and Fritz, Mario},
  journal={Advances in Neural Information Processing Systems},
  volume={37},
  pages={83548--83599},
  year={2024}
}

@inproceedings{qian2026collabbench,
  title={{CollabBench: Benchmarking and Unleashing Collaborative Ability of LLMs with Diverse Players via Proactive Engagement}},
  author={Qian, Hong and Liu, Yuanhao and Zhou, Zihan and Zhang, Zongbao and Ge, Hanjie and Shi, Haotian and Dou, Liang and Wang, Xiangfeng and Yang, Jing-Wen and Zhou, Aimin},
  booktitle={International Conference on Machine Learning},
  year={2026}
}

@inproceedings{du2024improving,
  title={{Improving Factuality and Reasoning in Language Models through Multiagent Debate}},
  author={Du, Yilun and Li, Shuang and Torralba, Antonio and Tenenbaum, Joshua B and Mordatch, Igor},
  booktitle={International Conference on Machine Learning},
  pages={11733--11763},
  year={2024},
  organization={PMLR}
}

@article{duan2024gtbench,
  title={{GTBench: Uncovering the Strategic Reasoning Capabilities of LLMs via Game-Theoretic Evaluations}},
  author={Duan, Jinhao and Zhang, Renming and Diffenderfer, James and Kailkhura, Bhavya and Sun, Lichao and Stengel-Eskin, Elias and Bansal, Mohit and Chen, Tianlong and Xu, Kaidi},
  journal={Advances in Neural Information Processing Systems},
  volume={37},
  pages={28219--28253},
  year={2024}
}

@book{schelling1980strategy,
  title={{The Strategy of Conflict: with a new Preface by the Author}},
  author={Schelling, Thomas C},
  year={1980},
  publisher={Harvard University Press}
}

@article{aharon2026tacit,
  title={{Tacit Coordination of Large Language Models}},
  author={Aharon, Ido and La Malfa, Emanuele and Wooldridge, Michael and Kraus, Sarit},
  journal={arXiv preprint arXiv:2601.22184},
  year={2026}
}

@inproceedings{riedl2026emergent,
  title={{Emergent Coordination in Multi-Agent Language Models}},
  author={Riedl, Christoph},
  booktitle={International Conference on Learning Representations},
  pages={120776--120799},
  year={2026}
}

@article{ashery2025emergent,
  title={{Emergent Social Conventions and Collective Bias in LLM Populations}},
  author={Ashery, Ariel Flint and Aiello, Luca Maria and Baronchelli, Andrea},
  journal={Science Advances},
  volume={11},
  number={20},
  pages={eadu9368},
  year={2025},
  publisher={American Association for the Advancement of Science}
}

@inproceedings{angluin1980local,
  title={{Local and Global Properties in Networks of Processors}},
  author={Angluin, Dana},
  booktitle={Proceedings of the twelfth annual ACM symposium on Theory of computing},
  pages={82--93},
  year={1980}
}

@article{patil2026randomness,
  title={{Randomness is Sometimes Necessary for Coordination}},
  author={Patil, Rohan and Malegaonkar, Jai and Christensen, Henrik I},
  journal={arXiv preprint arXiv:2605.06825},
  year={2026}
}

@article{agapiou2022melting,
  title={{Melting Pot 2.0}},
  author={Agapiou, John P and Vezhnevets, Alexander Sasha and Du{\'e}{\~n}ez-Guzm{\'a}n, Edgar A and Matyas, Jayd and Mao, Yiran and Sunehag, Peter and K{\"o}ster, Raphael and Madhushani, Udari and Kopparapu, Kavya and Comanescu, Ramona and others},
  journal={arXiv preprint arXiv:2211.13746},
  year={2022}
}

@article{papoudakis2021benchmarking,
  title={{Benchmarking Multi-Agent Deep Reinforcement Learning Algorithms in Cooperative Tasks}},
  author={Papoudakis, Georgios and Christianos, Filippos and Sch{\"a}fer, Lukas and Albrecht, Stefano V},
  journal={Proceedings of the Neural Information Processing Systems Track on Datasets and Benchmarks},
  year={2021}
}

@article{wu2021too,
  title={{Too many cooks: Bayesian inference for coordinating multi-agent collaboration}},
  author={Wu, Sarah A and Wang, Rose E and Evans, James A and Tenenbaum, Joshua B and Parkes, David C and Kleiman-Weiner, Max},
  journal={Topics in Cognitive Science},
  volume={13},
  number={2},
  pages={414--432},
  year={2021},
  publisher={Wiley Online Library}
}

@inproceedings{hu2021off,
  title={{Off-Belief Learning}},
  author={Hu, Hengyuan and Lerer, Adam and Cui, Brandon and Pineda, Luis and Brown, Noam and Foerster, Jakob},
  booktitle={International Conference on Machine Learning},
  pages={4369--4379},
  year={2021},
  organization={PMLR}
}
